\documentclass[a4paper, 11pt]{article}
\usepackage[sort&compress,square,comma,authoryear]{natbib}

\usepackage{a4wide}
\usepackage{algorithm}
\usepackage{algorithmic}
\usepackage[algo2e,linesnumbered,ruled]{algorithm2e}

\usepackage{amsfonts}
\usepackage{amsmath}
\usepackage{amssymb}
\usepackage{amstext}
\usepackage{amsthm}

\usepackage{booktabs}
\usepackage{dirtytalk}
\usepackage{enumitem}
\usepackage{hyperref}
\usepackage{mathtools}
\usepackage{microtype}
\usepackage{graphicx}
\usepackage{subfigure}
\usepackage{booktabs}
\usepackage[T1]{fontenc}
\usepackage{graphicx}

\usepackage{hyperref}
\usepackage[capitalize]{cleveref}

\usepackage{setspace}
\usepackage{wrapfig}

\usepackage{thmtools}
\usepackage{thm-restate}

\usepackage{nicefrac}
\usepackage{microtype}      
\usepackage[dvipsnames]{xcolor} 
\usepackage[parfill]{parskip}

\definecolor{salmon}{RGB}{234,153,153}
\definecolor{cornflowerblue}{RGB}{6,69,173}
\hypersetup{
    colorlinks,
    linkcolor={cornflowerblue},
    citecolor={cornflowerblue},
    urlcolor={salmon}
}

\usepackage{amsmath,amsfonts,amssymb,amsthm,bm}
\usepackage{booktabs}
\usepackage{dirtytalk}
\usepackage{enumitem}
\usepackage[T1]{fontenc}   
\usepackage{hyperref}      
\usepackage{booktabs}
\usepackage{graphicx}
\usepackage{mathrsfs}
\usepackage{mathtools}
\usepackage{microtype}
\usepackage{multirow}
\usepackage{tabularx}
\newcolumntype{Y}{>{\centering\arraybackslash}X}
\newcolumntype{Z}{>{\raggedleft\arraybackslash}X}
\usepackage{tikz}
\usetikzlibrary{calc, fadings, decorations, shapes, positioning, arrows}
\usepackage{thmtools}
\usepackage{todonotes}
\usepackage{url}
\usepackage{wrapfig}
\usepackage{xcolor,colortbl}
\usepackage{xspace}
\usepackage{multicol}
\usepackage{caption}

\arrayrulecolor{gray!65}

\newcommand{\acde}{\textsc{ACDE}\xspace}
\newcommand{\expect}[1]{\mathbb{E}  [ #1 ]}

\newcommand{\pa}[1]{\mathsf{Pa}(#1)}
\newcommand{\pr}[2]{\mathbb{P}_{#1} ( #2 )}
\DeclareMathOperator*{\argmin}{arg\,min}
\newcommand{\ffrac}[2]{\ensuremath{\frac{\displaystyle #1}{\displaystyle #2}}}

\theoremstyle{plain}

\newtheorem{theorem}{Theorem}[section]

\newtheorem{example}[theorem]{Example}
\theoremstyle{definition}
\newtheorem{definition}[theorem]{Definition}

\theoremstyle{remark}

\AfterEndEnvironment{theorem}{\noindent\ignorespaces}
\AfterEndEnvironment{proposition}{\noindent\ignorespaces}
\AfterEndEnvironment{lemma}{\noindent\ignorespaces}
\AfterEndEnvironment{corollary}{\noindent\ignorespaces}
\AfterEndEnvironment{definition}{\noindent\ignorespaces}
\AfterEndEnvironment{remark}{\noindent\ignorespaces}

\title{\textbf{DRCFS: Doubly Robust Causal Feature Selection}}

\date{}

\date{}

\usepackage{authblk}

\author[1]{Francesco Quinzan\thanks{This work was initiated and partially done while the author was employed at KTH Royal Institute of Technology.}${}^{\dagger}$}
\author[2]{Ashkan Soleymani${}^{\dagger}$}
\author[2]{Patrik Jaillet}
\author[3]{\\Cristian R. Rojas}
\author[4,5]{Stefan Bauer}
\affil[1]{Department of Computer Science, University of Oxford}
\affil[2]{Laboratory for Information \& Decision Systems (LIDS), Massachusetts Institute of Technology}
\affil[3]{KTH Royal Institute of Technology}
\affil[4]{TU Munich}
\affil[5]{Helmholtz Munich}
\affil[ ]{}
\affil[ ]{\small \textit{${}^{\dagger}$ equal contribution}}

\begin{document}

\maketitle

\begin{abstract}
\noindent Knowing the features of a complex system that are highly relevant to a particular target variable is of fundamental interest in many areas of science. Existing approaches are often limited to linear settings, sometimes lack guarantees, and in most cases, do not scale to the problem at hand, in particular to images. We propose DRCFS, a doubly robust feature selection method for identifying the causal features even in nonlinear and high dimensional settings. We provide theoretical guarantees, illustrate necessary conditions for our assumptions, and perform extensive experiments across a wide range of simulated and semi-synthetic datasets. DRCFS significantly outperforms existing state-of-the-art methods, selecting robust features even in challenging highly non-linear and high-dimensional problems. 
\end{abstract}
\section{Introduction}

We study the fundamental problem of \emph{causal} feature selection for non-linear models. That is, consider a set of features $\boldsymbol{X} = \{X_1, \dots, X_m\}$, and an outcome $Y$ specified with an additive-noise model on some of the features \citep{DBLP:conf/nips/HoyerJMPS08,DBLP:conf/icml/ScholkopfJPSZM12,DBLP:journals/jmlr/PetersMJS14}:
\begin{enumerate}[leftmargin=20mm,label={Axiom (\Alph*)},itemsep=0pt,topsep=0pt]
   \item \label{cond:1} $Y = f(\pa{Y}) + \varepsilon$,
\end{enumerate}
for a subset $\pa{Y}\subseteq \boldsymbol{X}$ and posterior additive noise $\varepsilon$. Our goal is to identify the set of relevant features $\pa{Y}$ from observations. Feature selection is an important cornerstone of high-dimensional data analysis \citep{liu2007computational, bolon2015feature, li2017feature, butcher2020feature}, especially in data-rich settings. By including only relevant variables and removing nuisance factors, feature selection allows us to build models that are simple, interpretable, and more robust \citep{yu2020causality, janzing2020feature}. Moreover, in many applications in science and industry, we are not only interested in predictive features but we also aim to identify causal relationships between them \citep{pearl2009causality, spirtes2000causation,murphy2001active}. 

Knowing the causal structure allows one to understand the potential effects of interventions on a system of interest, spanning various fields such as economics~\citep{varian2016causal}, biology~\citep{hu2018application}, medicine~\citep{mehrjou2021genedisco, lv2021causal}, software engineering~\citep{siebert2022applications}, agriculture~\citep{tsoumas2022evaluating}, and climate research \citep{doi:10.1126/sciadv.aau4996}. In general, it is impractical to infer the underlying causal graph solely from observational data, and in some cases, it may even be infeasible \citep{pearl2009causality, spirtes2000causation, chickering2004large, maclaren2019can}. Often, the focus is on specific target variables, such as stock returns in trading, genes associated with a specific phenotype or disease, or reduction of $CO_2$ emissions in environmental studies, rather than all the interactions between the many variables affecting the underlying system.

Capturing the causal relationships between different features and a particular target variable is highly non-trivial in practice. Existing approaches for causal feature selection often make unrealistic assumptions on the data generating process (DGP), or they simply do not scale to the problem at hand in terms of computational efficiency~\citep{yu2020causality} and/or statistical efficiency~\citep{yu2021unified}. Moreover, many of these approaches have only limited theoretical guarantees, and they have only been evaluated on noiseless data, far from the needs of practitioners in industry and science  \cite{kira1992practical,guyon2008practical}. 

To overcome these limitations, we propose a novel, doubly robust causal feature selection method (DRCFS) that significantly outperforms existing state-of-the-art approaches, especially in non-linear, noisy, and data-sparse settings. Our contributions are as follows:
\begin{enumerate}[label={$\bullet$},itemsep=0pt,topsep=0pt]
\item We propose DRCFS, a novel method for doubly robust feature selection, even in non-linear and high dimensional settings. In particular, our approach has guarantees for realistic cases, when there are cycles or hidden confounders between the features (see Figure \ref{fig:graph}). 
\item Based on our model description, we provide theoretical guarantees that substantiate our framework. In particular, we illustrate the necessity of our assumptions for causal feature selection and demonstrate that our approach is doubly robust while achieving $\sqrt{n}$- consistency.  
\item We provide a comprehensive experimental evaluation across various synthetic and semi-synthetic datasets, demonstrating that DRCFS significantly outperforms an extensive list of state-of-the-art baselines for feature selection.  
\end{enumerate}
\paragraph{Additional related work.}
Learning causal features has been investigated early on, and \citet{guyon2007causal} provides a broad review of these works. Since then, multiple other works have proposed different ideas on how to learn causal features from data ~\citep{cawley2008causal,paul2017feature,yu2021unified}. However, none of these works recover all the direct causal parents asymptotically or non-asymptotically, and they do not provide guarantees, especially for the cyclic or confounded setting.

Under the expense of additional assumptions such as faithfulness, there has been an orthogonal line of research aiming to infer the Markov equivalence class through iterative testing of d-separation statements using conditional independence test as done in the $\mathrm{PC}$-algorithm \citep{MasSchJan19,pearl2009causality,spirtes2000causation}.
    For a recent overview of causal feature selection approaches and their evaluation, we refer to \citet{yu2020causality}. Another approach of interest is to consider strong assumption on the causal structure of the DGP to achieve identifiability \citep{DBLP:conf/nips/HoyerJMPS08,DBLP:journals/jmlr/PetersMJS14}. This approach includes the works by \citet{DBLP:journals/jmlr/ShimizuHHK06} and \citet{DBLP:conf/uai/PetersMJS11}.
\section{Model Ddescription}
\label{sec:preliminarie}
\subsection{Causal structure} 
We study causal feature selection for a model as in \ref{cond:1}, using the language of probabilistic causality and structural causal models (SCMs, see \citet{pearlj, 10.1214/21-AOS2064}). Following \citet{10.1214/21-AOS2064}, we define the causal structure of a model in terms of direct causal effects among the variables (see Appendix \ref{appendix:scm} for the precise concepts introduced in this section). Direct causal effects are defined by distribution changes due to \emph{interventions} on the DGP. An intervention amounts to actively manipulating the generative process of a random variable $X_j$, without altering the remaining components of the DGP. For instance, in a randomized experiments an intervention could consist of giving patients a medical treatment. We specifically consider perfect interventions $X_j \gets x_j$, by which the post-interventional variable is set to a constant $X_j\equiv x_j$. We denote with $Y \mid do(x_j)$ the post-interventional outcome $Y$. Using this notation, a variable $X_j$ has a direct causal effect on the outcome $Y$ if intervening on $X_j$ can affect the outcome while keeping fixed the other variables $\boldsymbol{X}_j^c \coloneqq \{X_1, \dots, X_{j-1}, X_{j+1}, \dots, X_m\}$, i.e., there exists $x_j' \neq x_j$ such that
\begin{equation}
\label{eq:direct_effect}   
\pr{}{Y\mid do(x'_j,\boldsymbol{x}_j^c)} \neq \pr{}{Y\mid do(x_j,\boldsymbol{x}_j^c)}
\end{equation}
for some array $\boldsymbol{x}_j^c$ in the range of $\boldsymbol{X}_j^c$. Following \citet{10.1214/21-AOS2064}, we define the causal structure $\mathcal{G}$ of the DGP as a directed graph whose edges represent all direct causal effects among the variables.

We assume unique solvability of the SCM (see Definition 3.3 by \citet{10.1214/21-AOS2064} and Appendix \ref{appendix:uniquesolv}). By this assumption, the distribution on the observed variables $\boldsymbol{V} = \{X_1, \dots, X_m, Y\} $ is specified by a mixing measurable function of the form $\boldsymbol{V} = \boldsymbol g(\boldsymbol{U})$. The function $\boldsymbol g$ is uniquely defined by the structural equations of the model (see Appendix \ref{appendix:uniquesolv}). Here, $\boldsymbol{U}$ is a random vector of latent sources. According to our model, there may be potential hidden confounders for the observed variables. Furthermore, the underlying graph $\mathcal{G}$ may contain cycles. An example of a causal graph for a uniquely solvable SCM is presented in Figure \ref{fig:graph}. 
\begin{figure}[!t]
  \begin{minipage}[c]{0.4\textwidth}
  \centering
    \begin{tikzpicture}[node distance=10mm and 8mm, main/.style = {draw, circle, minimum size=0.6cm, inner sep=1pt}, >={triangle 45}]  
    \node[main, dotted] (1) {$\scriptstyle \varepsilon$};
    \node[main] (2) [right =of 1] {$\scriptstyle X_1$}; 
    \node[main] (3) [right =of 2] {$\scriptstyle X_3$}; 
    \node[main, dotted] (4) [above =of $(2)!0.5!(3)$] {$\scriptstyle U$}; 
    \node[main, thick] (5) [below =of $(2)!0.5!(3)$] {$\scriptstyle Y$};
    \node[main] (6) [right =of 3] {$\scriptstyle X_2$}; 
    \draw[->] (2) -- (5); 
    \draw[->] (3) -- (5); 
    \draw[->] (3) to [bend right] (6); 
    \draw[->] (6) to [bend right] (3); 
    \draw[->, dotted] (4) -- (2);
    \draw[->, dotted] (4) -- (3);
    \draw[->, dotted] (1) -- (5);
  \end{tikzpicture}
  \end{minipage}\hfill
  \begin{minipage}[c]{0.6\textwidth}
    \caption{
     Causal structure of a DGP as in \ref{cond:1}-\ref{cond:3}. The outcome is specified as $Y = f(X_1, X_3) + \varepsilon$, with $f$ a deterministic non-linear function, and $\varepsilon$ exogenous independent noise. $U$ is a latent variable that acts as a confounder for $X_1$ and $X_3$.
    } \label{fig:graph}
  \end{minipage}
\end{figure}
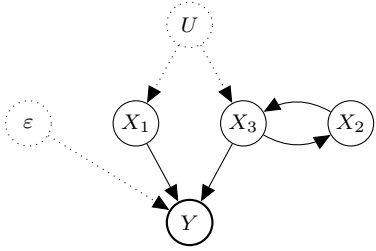
\subsection{Notation}
We always denote with $Y$ the outcome, with $X_j$ the features, for $j = 1, \dots, m$, and with $n$, the number of observations. We use the bold capital script, e.g. $\boldsymbol{X}, \boldsymbol{Z}$, to denote a group of random variables. The set $\boldsymbol{X}$ always refers to the set of all the features. For a given index $j = 1, \dots, m$, we denote with $\mathbf{X}^c_j$ the set consisting of all the features \emph{except} $X_j$. We use the standard $Y\mid do(\cdot)$ notation, to denote the post-interventional outcome. The notation $\pa{Y}$ denotes the subset of features that yield a direct causal effect on the outcome. We use the letter $\mathcal{D}$ to denote a generic dataset, and we denote with $\hat{\mathbb{E}}_{\mathcal{D}}[\ \cdot \ ]$ the empirical expected value over this dataset.
\subsection{Necessary Conditions for Causal Feature Selection}
We assume that the outcome is specified by the formula in \ref{cond:1}. Although very general, this model does not guarantee that we may identify $\pa{Y}$ from samples. In fact, there may be multiple models as in  \ref{cond:1} with different causal parents, that induce the same joint probability distribution on the observed variables. To overcome this problem, we consider the following additional assumptions on the data-generating process:
\begin{enumerate}[leftmargin=20mm,label={Axiom (\Alph*)},itemsep=0pt,topsep=0pt]
\setcounter{enumi}{1}
   \item \label{cond:2} $\varepsilon$ is exogenous noise, independent of $\boldsymbol{X}$.
   \item \label{cond:3} $Y$ has no direct causal effect on any of the features.
\end{enumerate}
\ref{cond:2} is a standard assumption in the literature of continuous additive noise models \citep{DBLP:journals/jmlr/PetersMJS14}. \ref{cond:3} requires that the outcome $Y$ does not causally affect any of the features $X_j$. This requirement is aligned with previous related work \citep{soleymani2022causal}. As we show in Section \ref{sec:causal_feature_selection}, under Axioms \ref{cond:1}-\ref{cond:3} we can identify the causal parents of the outcome from samples. However, if either \ref{cond:2} or \ref{cond:3} fail, then performing Causal Feature Selection from observations is impossible. We show this, by considering two counterexamples, given in Example \ref{example1} and Example \ref{example2} respectively.
\begin{example}[Necessity of \ref{cond:2}]
\label{example1}
Consider two datasets $\{X_i, Y_i\}$ for $i = 1,2$. The respective DGPs are defined as follows:
\begin{equation*}
\left \{
\begin{array}{l}
X_1 = N_{1} \\
\ Y_1 = \varepsilon_1
\end{array}
\right .
\quad
\text{and} \quad \left \{
\begin{array}{l}
X_2 = N_{2} \\
\ Y_2 = X_2 + \varepsilon_2
\end{array}
\right .
\end{equation*}
with $(N_1, \varepsilon_1) \sim \mathcal{N}(\boldsymbol 0, \Sigma)$, a zero-mean joint Gaussian distribution with covariance matrix
\begin{equation*}
\Sigma = \left [
\begin{array}{cc}
1 & 1 \\
1 & 1
\end{array}
\right ],
\end{equation*}
while $N_2 \sim \mathcal{N}(0, 1)$ and $\varepsilon_2 \equiv 0$. Note that $\{X_1, Y_1\}$ satisfies~\ref{cond:3} but violates~\ref{cond:2}, since the noise $\varepsilon_1$ is correlated with $X_1$. Both datasets entail the same joint probability distribution. However, $Y_1$ is exogenous to the model, whereas $X_2$ yields a direct causal effect on $Y_2$.
\end{example}
Example \ref{example1} shows that if \ref{cond:2} is violated, then it is impossible to perform Causal Feature Selection from observational data. We provide a second example, to show that \ref{cond:3} is also necessary.
\begin{example}[Necessity of \ref{cond:3}]
\label{example2}
Consider two datasets $\{X_i, Y_i\}$ for $i = 1,2$, with DGPs defined as:
\begin{equation*}
\left \{
\begin{array}{l}
X_1 = N_1 \\
\ Y_1 = X_1 + \varepsilon_1
\end{array}
\right .
\quad
\text{and} \quad \left \{
\begin{array}{l}
X_2 = 0.5 \cdot Y_2 + N_2\\
\ Y_2 = \varepsilon_2
\end{array}
\right .
\end{equation*}
where $\varepsilon_1, N_1, \sim \mathcal{N}(0, 1)$, $\varepsilon_2 \sim \mathcal{N}(0,2)$ and $N_2 \sim \mathcal{N}(0,0.5)$ are independent. In this case, dataset $\{X_2, Y_2\}$ satisfies~\ref{cond:2} but violates~\ref{cond:3}, since $Y_2$ causes $X_2$. Then, both datasets are jointly normal with zero mean and the same covariance matrix. However, $Y_1$ has as causal parent $X_1$, and $Y_2$ has no causal parents.
\end{example}
\section{Methodology}
\subsection{Double Machine Learning (DoubleML)}
\label{sec:riesz}\label{sec:mixed_bias_property}
We provide a statistical test for Causal Feature Selection based on DoubleML estimators. DoubleML is a general framework for parameter estimation, which uses debiased score functions and (double) cross-fitting to achieve $\sqrt{n}$-consistency guarantees. Following \citet{10.1093/biomet/asaa054,DBLP:conf/icml/ChernozhukovNQS22}, we provide a suitable debiased score function using the Riesz Representation Theorem and the Mixed Bias Property.

\textbf{The Riesz Representation Theorem.}%
Let $\boldsymbol{V}$ be the collection of observed random variables of an SCM as in Definition \ref{def:SCM}. Let $\boldsymbol{X}$ be a random variable in $\boldsymbol{V}$, $g$ any real-valued function of $\boldsymbol{X}$ such that $\expect{g^2(\boldsymbol{X})} < \infty$, and consider a linear functional $m(\boldsymbol{V}; g)$ in $g$. The Riesz Representation Theorem ensures that, under certain conditions,
there exists a function $\alpha_0$ of $\boldsymbol{X}$ such that $\expect{m(\boldsymbol{V}; g)} = \expect{\alpha_0(\boldsymbol{X}) g(\boldsymbol{X})}$. The function $\alpha_0$ is called the \emph{Riesz Representer} (RR). Crucially, the RR only depends on the functional $m$, and not on the function $g$. \citet{automatic_debiased} shows that the Riesz representer can be estimated by solving the following optimization problem:
\begin{equation}
\label{eq:riesz_estimation}
\alpha_0 = \argmin_\alpha  \expect{\alpha(\boldsymbol{X})^2 - 2m(\boldsymbol{V}; \alpha)}.
\end{equation}
Using the RR, we will derive a debiased score function for the parameter $\theta_0 \coloneqq \expect{m(\boldsymbol{V}; g_0)}$ with $g_0(\boldsymbol{X}) = \expect{Y\mid \boldsymbol{X}}$. This score function is \say{debiased} in the sense that it fulfills the Mixed Bias Property.

\textbf{The Mixed Bias Property.} The RR is crucial when learning an estimate $\hat{\theta}_0$ of the parameter $\theta_0 \coloneqq \expect{m(\boldsymbol{V}; g_0)}$ as described above. In fact, as shown by \citet{DBLP:conf/icml/ChernozhukovNQS22,10.1093/biomet/asaa054}, one can achieve $\sqrt{n}$-consistency for $\hat{\theta}_0$ by combining (double) cross-fitting with a debiased learning objective of the form
\begin{equation}
\label{eq:orthogonal_score}
\varphi( \theta, \boldsymbol \eta) \coloneqq m(\boldsymbol{V}; g ) + \alpha(\boldsymbol{X})\cdot(Y - g(\boldsymbol{X})) - \theta.
\end{equation}
Here, $\boldsymbol \eta \coloneqq (\alpha, g)$ is a nuisance parameter consisting of a pair of square-integrable functions. Note that it holds $\expect{\varphi( \theta_0, \boldsymbol \eta_0)} = 0$, with $\boldsymbol \eta_0 \coloneqq ( \alpha_0, g_0 )$ consisting of the RR $\alpha_0$ of the moment functional $m(\boldsymbol{V}; g)$ and the conditional expected value $g_0(\boldsymbol{X}) = \expect{Y\mid \boldsymbol{X}}$. As shown by \citet{DBLP:conf/icml/ChernozhukovNQS22}, the score function Eq. \eqref{eq:orthogonal_score} yields
\begin{equation*}
\expect{\varphi(\theta_0, \boldsymbol \eta)}
= - \expect{(\alpha(\boldsymbol{X})- \alpha_0(\boldsymbol{X}))(g(\boldsymbol{X})- g_0(\boldsymbol{X}))}.
\end{equation*}
This equation corresponds to the Mixed Bias Property as in Definition 1 of \citet{10.1093/biomet/asaa054}. Hence, when cross-fitting or double-cross is employed, the resulting estimator $\hat{\theta}$ has the \emph{double robustness property}. That is, the quantity $\sqrt{n}(\hat{\theta} - \theta_0)$ converges in distribution to a zero-mean Normal distribution whenever the product of the mean-square convergence rates of $\alpha$ and $g$ is faster than $\sqrt{n}$. By these guarantees, if an estimator of $g_0$ has a good convergence rate, then the rate requirement on the estimator of the RR is less strict, and vice versa.
\subsection{The Average Controlled Direct Effect (\acde{})}
\label{sec:causal_feature_selection}
Our approach to causal feature selection uses the \acde. The \acde is a concept introduced by \citet{DBLP:conf/uai/Pearl01} to provide an operational description of the statement \say{the variable $X_j$ directly influences the outcome $Y$}. By this conceptualization, we ask whether the expected outcome would change under an intervention $X_j \gets x_j$, while
holding the remaining observed variables at a predetermined value. Following the notation given in Section \ref{sec:preliminarie}, the \acde is defined as
\begin{equation}
\label{eq:acde}
    \acde(x_j, x'_j| \boldsymbol{x}_j^c) \coloneqq \expect{Y\mid do(x_j,\boldsymbol{x}_j^c)} - \expect{Y\mid do(x'_j,\boldsymbol{x}_j^c)}.
\end{equation}
This quantity captures the difference between the distribution of the outcome $Y$ under $X_j \gets x_j$ and $X_j \gets x'_j$, when we keep all other variables $\boldsymbol{X}_j^c$ fixed at some values $\boldsymbol{x}_j^c$. As we will show later, under Axioms \ref{cond:1}-\ref{cond:3} a variable $X_j$ is a causal parent of $Y$ if and only if the \acde is non-zero for at least a triple $(x_j, x'_j, \boldsymbol{x}_j^c)$ of possible interventional values. Hence, we can perform causal feature selection by testing whether the \acde is identically zero or not.
\begin{algorithm*}[t]
	\caption{Doubly Robust Causal Feature Selection Algorithm (DRCFS)}
	\label{alg}
	\begin{algorithmic}[1] 
	\STATE split the $n$ samples $\mathcal{D} = \{(x_{1i}, \dots, x_{mi}, y_i)\}_{i = 1, \dots, n}$ into $k$ disjoint sets $\mathcal{D}_1,\dots, \mathcal{D}_k$;
    \STATE define $\mathcal{D}_{l}^c \gets \mathcal{D} \setminus \mathcal{D}_{l}$ for all $l \in [k]$;\vspace{8pt}\\
    \FOR{each $l \in [k]$}
	\STATE construct an estimator $\hat{g}_{ l} (\boldsymbol{X})$ for $\expect{Y\mid \boldsymbol{X}}$ and an estimator $\hat{\alpha}_{l}(\boldsymbol{X})$ for the RR of $m_0$,  using samples in $\mathcal{D}_{l}^c$;
    \STATE $\hat{\theta}_{0,l} \gets \hat{\mathbb{E}}_{ \mathcal{D}_{l}}  [Y \cdot \hat{g}_{l}(\boldsymbol{X}) - Y \cdot \hat{\alpha}_{l}(\boldsymbol{X}) - \hat{\alpha}_{l }(\boldsymbol{X}) \cdot \hat{g}_{l }(\boldsymbol{X}) ] $;
    \STATE $\hat{\sigma}_{0,l}^2 \gets \hat{\mathbb{E}}_{ \mathcal{D}_{l}}  [ (Y \cdot \hat{g}_{l}(\boldsymbol{X}) - Y \cdot \hat{\alpha}_{l}(\boldsymbol{X}) - \hat{\alpha}_{l }(\boldsymbol{X}) \cdot \hat{g}_{l }(\boldsymbol{X}) - \hat{\theta}_{0,l} )^2 ] $;
    \ENDFOR
    \STATE $\hat{\theta}_{0} \gets \frac{1}{k} \sum_{l }\hat{\theta}_{0,l}$ and $\hat{\sigma}_{0}^2 \gets \frac{1}{k} \sum_{l}\hat{\sigma}_{0,l}^2$;\vspace{8pt}\\
	\FOR{each feature $X_j \in \boldsymbol{X}$}
    \FOR{each $l \in [k]$}
    \STATE construct an estimator $\hat{h}_{l}^j(\boldsymbol{X}_j^c) $ for $\expect{Y\mid \boldsymbol{X}_j^c}$ and an estimator $\hat{\alpha}_{l}^j(\boldsymbol{X}_j^c)$ for the RR of $m_0^j$, using samples in $\mathcal{D}_{l}^c$;
    \STATE $\hat{\theta}_{j,l} \gets \hat{\mathbb{E}}_{ \mathcal{D}_{l}}  [Y \cdot \hat{h}_{l}^j(\boldsymbol{X}_j^c) - Y \cdot \hat{\alpha}_{l}^j(\boldsymbol{X}_j^c) - \hat{\alpha}_{l }^j(\boldsymbol{X}_j^c) \cdot \hat{h}_{l }^j(\boldsymbol{X}_j^c)] $;
    \STATE $\hat{\sigma}_{j,l}^2 \gets \hat{\mathbb{E}}_{ \mathcal{D}_{l}}  [ (Y \cdot \hat{h}_{l}^j(\boldsymbol{X}_j^c) - y \cdot \hat{\alpha}_{l}^j(\boldsymbol{X}_j^c) - \hat{\alpha}_{l }^j(\boldsymbol{X}_j^c) \cdot \hat{h}_{l }^j(\boldsymbol{X}_j^c) - \hat{\theta}_{j,l}^j )^2]$;
    \ENDFOR
    \STATE $\hat{\theta}_{j} \gets \frac{1}{k} \sum_{l }\hat{\theta}_{j,l}$ and $\hat{\sigma}_{j}^2 \gets \frac{1}{k} \sum_{l }\hat{\sigma}_{j,l}^2$;\vspace{8pt}\\
    \STATE perform a paired $t$-test to determine if $\hat{\theta}_{j} \approx \hat{\theta}_{0}$ and select feature $X_j$ if the null-hypotheses is rejected.
    \ENDFOR
   	\STATE \textbf{return} the selected features $X_j$;
   	\end{algorithmic}
\end{algorithm*}
\paragraph{Using the \acde{} for Causal Feature Selection.} Our approach to causal feature selection essentially consists of testing whether a feature $X_j$ yields a non-zero average controlled direct effect on the outcome. This approach is justified by the following result. 
\begin{restatable}{lemma}{lemmaparents}
\label{lemma:parents}
Consider a causal model as in Axioms \ref{cond:1}-\ref{cond:3}, and fix a feature $X_j$. Then, $\expect{Y\mid do(x_j,\boldsymbol{x}_j^c)} - \expect{Y\mid do(x'_j,\boldsymbol{x}_j^c)} \neq 0$ for some interventional values $x_j, x'_j $ and $\boldsymbol{x}_j^c$ if and only if $X_j \in \mathsf{Pa}(Y)$.
\end{restatable}
The proof is deferred to Appendix \ref{lemmaparents}. Intuitively, this proof tells us that direct causal effects can be checked with the expected value of the post-interventional outcome, instead of the full distribution as in Eq. \eqref{eq:direct_effect}. This lemma crucially relies on Axioms \ref{cond:1}-\ref{cond:3} and it does not hold for general causal models.

\textbf{Estimating the \acde{} from samples.} A major challenge in using the \acde for Causal Feature Selection is that we do not assume knowledge of the post-interventional outcome distribution. However, it turns out that we can perform this test from observational data, by considering this quantity:
\begin{equation}
\label{eq:chi}
    \chi_j \coloneqq \mathbb{E}_{(x_j, \boldsymbol{x}_j^c) \sim (X_j, \boldsymbol{X}_j^c)}\left [\left ( \expect{Y | x_j, \boldsymbol{x}_j^c} - \expect{Y | \boldsymbol{x}_j^c} \right )^2\right].
\end{equation}
Under mild assumptions, testing whether $\chi_j = 0$ is equivalent to testing if there is no average direct effect.
\begin{restatable}{lemma}{lemmachi}
\label{lemma:chi}
Consider a causal model as in Axioms \ref{cond:1}-\ref{cond:3}. Then, $\expect{Y\mid do(x_j,\boldsymbol{x}_j^c)} - \expect{Y\mid do(x'_j,\boldsymbol{x}_j^c)} \neq 0$ almost surely if and only if $\chi_j \neq 0$. It follows that $\chi_j \neq 0$ if and only if $X_j \in \pa{Y}$.
\end{restatable}
The proof is deferred to Appendix \ref{lemmachi}. From this lemma, we can use $\chi_j $ to estimate if a feature $X_j$ is a causal parent of the outcome $Y$, i.e., perform Causal Feature Selection, by testing if $\chi_j = 0$. Crucially, $\chi_j$ can be estimated efficiently from samples.
\subsection{Estimating the \acde{} with DoubleML}
\label{sec:debiased}
We provide a debiased score function for $\chi_j$, using the Riesz representation theorem, as described in Section \ref{sec:riesz}. Crucially, we show that $\chi_j$ is the difference of the expected value of two linear moment functionals. The following lemma holds.
\begin{restatable}{lemma}{condvariance}
\label{lemma:cond_variance}
Let $\chi_j$ be as in Eq. \eqref{eq:chi}, and $\boldsymbol{X} = \{X_j\} \cup \boldsymbol{X}_j^c$. For any square-integrable random variable $g(\boldsymbol{X})$, consider the moment functional $m_0(\boldsymbol{V}; g) \coloneqq Y  g(\boldsymbol{X})$. Similarly, for any square-integrable random variable $h(\boldsymbol{X}_j^c)$, consider the moment functional $m_j(\boldsymbol{V}; h) \coloneqq Y  h(\boldsymbol{X}_j^c)$.\footnote{Note that $m_0(\boldsymbol{V}; g)$ and $m_j(\boldsymbol{V}; h)$ are distinct functionals, since they are defined over sets of functions with different domains.} Then, it holds that 
\begin{equation*}
    \chi_j = \expect{m_0(\boldsymbol{V}; g_0)} - \expect{m_j(\boldsymbol{V}; h_0)},
\end{equation*}
with $g_0(\boldsymbol{X}) = \expect{Y \mid \boldsymbol{X}}$ and $h_0(\boldsymbol{X}_j^c) = \expect{Y \mid \boldsymbol{X}_j^c}$.
\end{restatable}
The proof is deferred to Appendix \ref{appendix:cond_variance}. By this lemma, we can estimate the parameter $\chi_j$, by learning the parameters $\theta_0 \coloneqq \expect{m_0(\boldsymbol{V}; g_0)}$ and $ \theta_j \coloneqq \expect{m_j(\boldsymbol{V}; h_0)}$ separately and them taking their difference. Both $\theta_0$ and $\theta_j$ are expected values of linear moment functionals, so we can apply DoubleML as described in Section \ref{sec:riesz} to obtain $\sqrt{n}$-consistent estimates for them. The resulting estimators $\hat{\theta}_0$ and $\hat{\theta}_j$ have the double robustness property, and so does their difference $\chi_j = \hat{\theta}_0-\hat{\theta}_j$. We can then take advantage of the fast convergence rate, to determine if $\hat{\theta}_0 \approx \hat{\theta}_j$ with a paired $t$-test.
\section{The DRCFS Algorithm}
\label{sec:approach}
\subsection{Overview} Our approach to causal discovery essentially consists of testing whether a feature $X_j$ yields a non-zero average controlled direct effect on the outcome, following Lemma \ref{lemma:parents}. We refer to our approach as the Doubly Robust Causal Feature Selection Algorithm (DRCFS, see Algorithm \ref{alg}). This algorithm consists of the following steps:
\begin{enumerate}[label={$\bullet$},itemsep=0pt,topsep=0pt]
    \item Select a variable $X_j$ to test if $X_j \in \pa{Y}$.
    \item Estimate the parameter $\chi_j$ using DoubleML, as described in Section \ref{sec:debiased}. The resulting estimator $\hat{\chi}_j$ has the double-robustness property.
    \item By Lemmas~\ref{lemma:chi} and \ref{lemma:parents}, the variable $X_j$ is not a parent of $Y$ if and only if $\chi_j = 0$. Use a paired $t$-test for $\hat{\chi}_j $ to select or discard $X_j$ as a parent of $Y$.
\end{enumerate}
This procedure can be iterated for each of the features. Crucially, the estimator $\hat{\chi}_j$ ought to have the double-robustness property. To this end, we resort to Lemma \ref{lemma:cond_variance}. We first estimate the parameters $\theta_0 \coloneqq \expect{m_0(\boldsymbol{V}, g_0)}$ and $ \theta_j \coloneqq \expect{m_j(\boldsymbol{V}, h_0)}$ separately using DoubleML, and then obtain the desired estimator $\hat{\chi}_j$ by taking the difference.
\subsection{Double-Robustness of $\hat{\chi}_j$}
We now show that the estimand $\hat{\chi}_j$ as in Line 16 of Algorithm \ref{alg} has the double-robustness property. To this end, we show that $\hat{\theta}_0$ and $\hat{\theta}_j$ in Algorithm \ref{alg} have the double robustness property. We focus on $\hat{\theta}_0$ since the case for $\hat{\theta}_j$ is analogous. To this end, we can apply the Riesz representation theorem, as described in Section \ref{sec:riesz}, to obtain that $\theta_0 = \expect{\alpha_0(\boldsymbol{X})\cdot m_0(\boldsymbol{V}; g_0)}$, with $\alpha_0 $ the Riesz representer of the moment functional, and $g_0$ the conditional expected value. 

We can use DoubleML as in Section \ref{sec:mixed_bias_property} to derive a debiased score function for the term $\theta_0$. Consider a dataset of $n$ samples $\mathcal{D} = \{(x_{1i}, \dots, x_{mi}, y_i)\}_{i = 1, \dots, n}$ and let $\mathcal{D}_1,\dots, \mathcal{D}_k$ be a disjoint $k$-partition of this dataset. Following Eq. \eqref{eq:orthogonal_score}, we provide an estimator $\hat{\theta}_{0,l}$ for $\theta_0$ on the partition $\mathcal{D}_l$ as
\begin{equation}
\label{eq:new_eq}
\hat{\theta}_{0,l} = \hat{\mathbb{E}}_{ \mathcal{D}_{l}}  [Y \cdot \hat{g}_{l}(\boldsymbol{X}) - Y \cdot \hat{\alpha}_{l}(\boldsymbol{X}) - \hat{\alpha}_{l }(\boldsymbol{X}) \cdot \hat{g}_{l }(\boldsymbol{X}) ] 
\end{equation}
Here, $\hat{\mathbb{E}}_{ \mathcal{D}_{l}}$ denotes the empirical expected value over $\mathcal{D}_{l}$. The function $\hat{\alpha}_l$ is an estimator for the Riesz representer obtained by Eq. \eqref{eq:riesz_estimation} over the complementary sample set $\mathcal{D}_l^c = \mathcal{D} \setminus \mathcal{D}_l$. Note that Eq. \eqref{eq:new_eq} corresponds to Line 5 of Algorithm \ref{alg}. Similarly, $\hat{g}_l$ is an estimator for the conditional expected value, which is obtained as a solution of the $\ell_1$-regularized regression problem \citep{10.1093/biomet/asaa054} over the sample set $\mathcal{D}_l^c = \mathcal{D} \setminus \mathcal{D}_l$. We can also estimate the variance over the samples as in Line 6 of Algorithm \ref{alg}. Estimates for the parameters $\theta_0$ and the variance are then 
\begin{equation*}
\hat{\theta}_{0} = \frac{1}{k} \sum_{l = 1}^k\hat{\theta}_{0,l} \quad \text{and} \quad \hat{\sigma}_{0}^2 = \frac{1}{k} \sum_{l = 1}^k\hat{\sigma}_{0,l}^2.
\end{equation*}
These estimators are given in Line 8 of Algorithm \ref{alg}, and they have the double-robustness property, i.e., 
\begin{equation*}
\sqrt{n}(\hat{\theta}_{0} - \theta_{0})\rightsquigarrow \mathcal{N}(0, \sigma^2_0).
\end{equation*}
Furthermore, the empirical variance $\hat{\sigma}_{\chi_j}^2$ is a $\sqrt{n}$-consistent estimator for $\sigma^2_0$. Similarly, the parameters $\hat{\theta}_{j}$ and $ \hat{\sigma}_{j}^2$ as in Line 15 of Algorithm \ref{alg} are doubly-robust.
\subsection{Statistical Test} For a given confidence interval, testing whether $\hat{\chi}_j \approx 0 $, is equivalent to testing whether the estimates of $\hat{\theta}_j$ and $\hat{\theta}_0$ have the same mean. To accomplish this, we perform paired sample t-tests. Given the presence of multiple dependent tests, in case the data is highly dependent, it is important to control for false discovery rate (FDR) in order to accurately assess the results. For this purpose, we apply the Benjamini-Yekutieli procedure~\citep{benjamini2001control}.
\section{Experiments}
In this section, we evaluate the performance of our algorithm extensively. First, we show the superiority of our method compared to well-established algorithms, using synthetic data consisting of causal structures created by various DGPs. Second, we demonstrate a real-world application of DRCFS on microbiome abundance. Lastly, we further showcase the performance and scalability of our algorithm on bnlearn benchmarks~\citep{bnlearn}.
\subsection{Causal Feature Selection for Synthetic Data}
Here, we discuss the data generating process, comparison of the performance with the baselines, robustness of performance w.r.t. various characteristics of the underlying causal structure such as connectivity level, and quality of estimations and statistical tests.

\textbf{Data Generating Process.}
The DGP generally follows the same procedure as~\citet{soleymani2022causal} to produce direct acyclic graphs (DAGs): (1) Nodes are randomly permutated to form a topological order. (2) For each pair of nodes $X_i$ and $X_j$, where $X_i$ precedes $X_j$ in this order, an edge $X_i \rightarrow X_j$ is added to the graph with probability $p_c$ (connectivity level). (3) The values are assigned to each variable $X$ as a transformation $f$ of the direct causes of that variable, plus posterior additive noise as in Axioms \ref{cond:1}-\ref{cond:3} for each variable $X^\prime \in \boldsymbol{X} \cup \{Y\}$. (4) Each node $X \in \boldsymbol{X}$ is concealed to serve as a unseen confounder with probability $p_h$. The parameters of interest in the DGP are: number of nodes $m$, number of observations $n$, connectivity level $p_c$, transformation function $f$, random variable $\epsilon$ representing the additive noise, probability of hiding each node $p_h$ to serve as unseen confounders. For convenience, different choices of function $f$ are given in~\Cref{sec:design_f}. Let $\{f_1, f_2, \dots, f_k\}$ and $\{\pi_1, \pi_2, \dots, \pi_k\}$ be the potential choices of $f$ and their corresponding probabilities, then $f$ is chosen by $f \sim \sum_{i = 1}^{k} \pi_i \delta_{f_i}$, for each variable $X^\prime \in \boldsymbol{X} \cup \{Y\}$, where $\delta$ is the Dirac probability measure.

\textbf{Baselines.} We compare the performance of our method with a diverse set of causal structure learning and inference for regression algorithms: \textsc{LiNGAM}~\citep{shimizu2006linear}, order-independent \textsc{PC}~\citep{JMLR:v15:colombo14a}, rankPC, rankFCI~\citep{spirtes2000causation, heinze2018causal}, MMHC~\citep{mmhc}, GES~\citep{ges}, rankGES, ARGES (adaptively restricted GES~\citep{arges}), rankARGES, FCI+~\citep{fciplus}, PCI~\citep{shah2020hardness}, CORTH Features~\citep{soleymani2022causal}, Standard Linear Regression, Lasso with exact post-selection inference~\citep{lee2016exact}, Debiased Lasso~\citep{javanmard2015biasing}, Forward Stepwise Regression for active variables~\citep{loftus2014significance,tibshirani2016exact}, Forward Stepwise Regression for all variables~\citep{loftus2014significance,tibshirani2016exact}, LARS for active variables~\citep{efron2004least,tibshirani2016exact}, and LARS for all variables~\citep{efron2004least,tibshirani2016exact}. R Packages "CompareCausalNetworks"\footnote{ \url{https://cran.r-project.org/web/packages/CompareCausalNetworks/index.html}} and "selectiveInference: Tools for Post-Selection Inference"\footnote{ \url{ https://cran.r-project.org/web/packages/selectiveInference/}} are used for a great number of the baselines. 
\begin{figure}[!t] 
  \centering \includegraphics[width=\columnwidth]{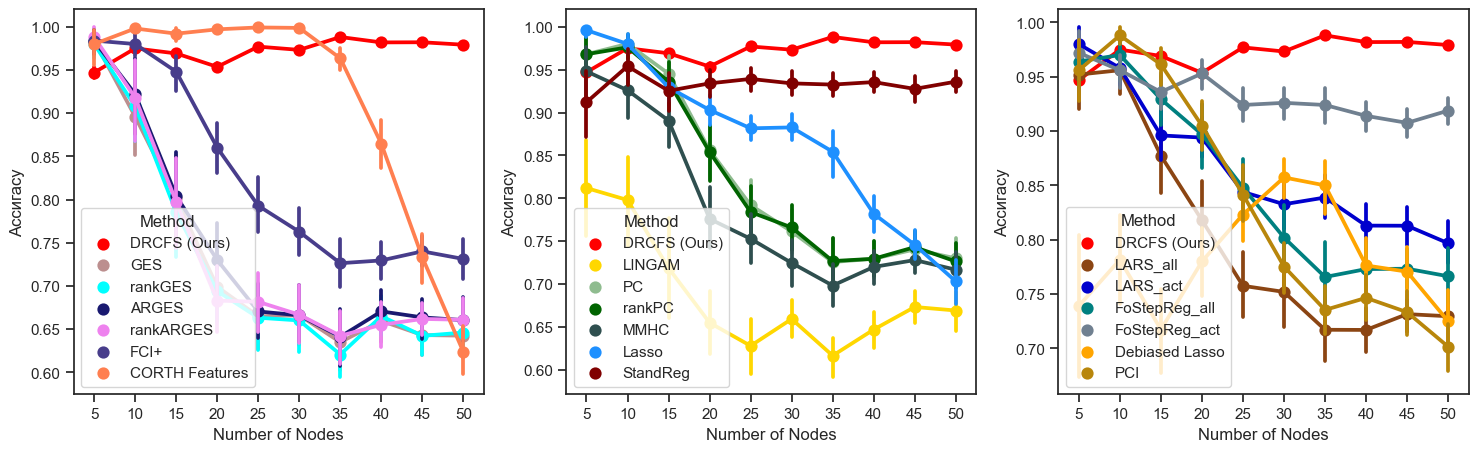}
  \caption{Performance (accuracy) of the algorithms w.r.t. number of nodes $m$ for fully linear causal structures ($f = f_1$ with probability 1), where $p_s = 0.3$, $p_h = 0$, and $\epsilon \sim \mathcal{N}(0,1)$. Each case is averaged over 50 simulations. We use ForestRiesz with identity feature map $\phi(\boldsymbol{X}) = \boldsymbol{X}$. DRCFS's performance shows stability even for large graphs while the baselines suffer in high dimensions. Plots for other metrics are provided in \Cref{sec:add_plots}, \Cref{fig:exp1_csi,fig:exp1_f1}.}
  \label{fig:exp1_acc}
\end{figure}
\begin{figure}[!t] 
  \centering \includegraphics[width=\columnwidth]{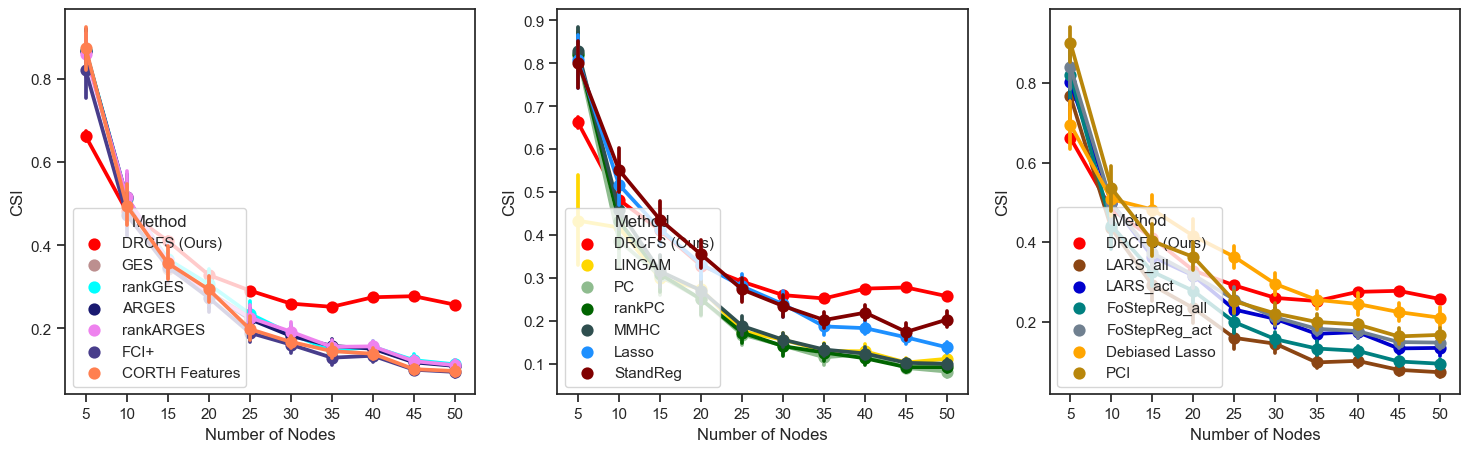}
  \caption{Performance (CSI) of the algorithms w.r.t. number of nodes $m$ for fully Log-sum-exp causal structures ($f = f_6$ with probability 1), where $p_s = 0.5$, $p_h = 0.1$, and $\epsilon \sim \mathcal{N}(0,1)$. Each case is averaged over 50 simulations. We use ForestRiesz with identity feature map $\phi(\boldsymbol{X}) = \boldsymbol{X}$. DRCFS's performance shows stability even for large graphs while the baselines suffer in high dimensions. Plots for other metric are provided in \Cref{sec:add_plots}, \Cref{fig:exp4_f1}.}
  \label{fig:exp4_csi}
\end{figure}
\textbf{Evaluation.} We use accuracy, F1 score, and CSI (discussed in~\Cref{sec:eval_metrics}) as our evaluation metrics. We consider F1 score and CSI because they put emphasis on the number of true positives.
\textbf{Estimation Technique. } In this section, we leverage a special case of Generalized Random Forests~\citep{athey2019generalized}, called ForestRiesz~\citep{automatic_debiased} to estimate the conditional expected value and RR as in Lemma \ref{lemma:cond_variance}.\footnote{Another candidate estimator can be utilized to learn conditional expected value and RR under the presence of complicated structures is RieszNet introduced by~\citet{automatic_debiased}. } We illustrate this technique, by focusing on $g_0$ and $\alpha_0$, since the case for $g_j$, $\alpha_j$ is analogous. Assume that $g_0 $ and $\alpha_0 $ are locally linear with respect to a smooth feature map $\phi$, i.e., $g_0(\boldsymbol{X}) = \langle \phi(\boldsymbol{X}), \beta(\boldsymbol{X})\rangle$ and $\alpha_0(\boldsymbol{X}) = \langle \phi(\boldsymbol{X}), \gamma(\boldsymbol{X})\rangle$ with $\beta(\boldsymbol{X})$ and $\gamma(\boldsymbol{X})$ non-parametric estimators derived by the tree. Then, in order to learn $g_0$, we minimize the square loss $\mathcal{R}(\beta) = \mathbb{E}_n [(\langle \phi(\boldsymbol{X}), \beta(\boldsymbol{X})\rangle - Y)^2]$. This is equivalent to searching for the solution of the equation $$\mathbb{E}_n [(\langle \phi(\boldsymbol{X}), \beta(\boldsymbol{X})\rangle - Y) \phi(\boldsymbol{X})\:|\: \boldsymbol{X}] = 0$$
in the variable $\beta$. Similarly, we learn the RR by minimizing the score $\mathcal{R}(\gamma) = \mathbb{E}_n[\langle \phi(\boldsymbol{X}), \gamma(\boldsymbol{X}) \rangle^2 - 2 \gamma(\boldsymbol{X})^T m(\boldsymbol{V}; \phi(\boldsymbol{X})) ]$. This amounts to solving
\begin{equation*} 
\mathbb{E}_n[\phi(\boldsymbol{X}) \phi(\boldsymbol{X})^T \gamma(\boldsymbol{X}) - m(\boldsymbol{V}; \phi(\boldsymbol{X})) \:|\: \boldsymbol{X}] = 0.
\end{equation*}
Thus, our approach is in line with the works by~\citet{athey2019generalized, automatic_debiased}. 

In order to effectively ensure fair evaluations against baselines, and to account for efficiency, the feature map used in our evaluations within this section is set to the identity function $\phi(\boldsymbol{X}) = \boldsymbol{X}$. However, it should be noted that alternative feature maps may be employed, to cater to specific requirements and prior domain knowledge of the dataset. 
\begin{figure}[!t] 
  \centering \includegraphics[width=\columnwidth]{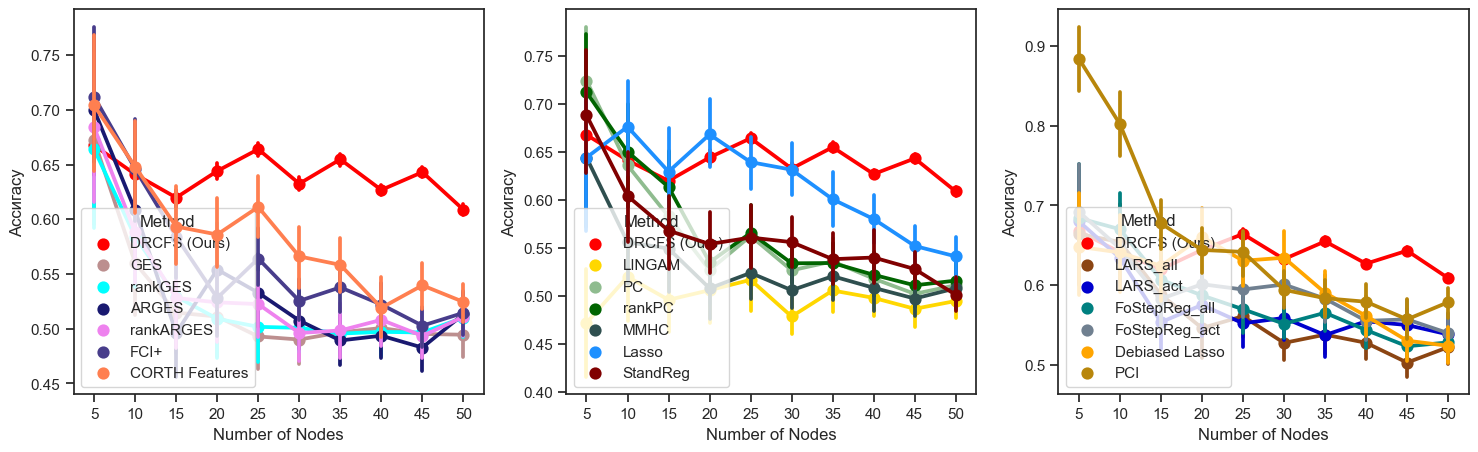}
  \caption{Performance (accuracy) of the algorithms w.r.t. number of nodes $m$ for causal structures with geometric mean relationship ($f = f_5$ with probability 0.8, $f = f_1$ with probability 0.2), where $p_s = 0.5$, $p_h = 0$, and $\epsilon \sim \mathcal{N}(0,1)$. Each case is averaged over 50 simulations. We use ForestRiesz with identity feature map $\phi(\boldsymbol{X}) = \boldsymbol{X}$. DRCFS's performance shows stability even for large graphs while the baselines suffer in high dimensions. Plots for other metric are provided in \Cref{sec:add_plots}, \Cref{fig:exp3_csi,fig:exp3_f1}.}
  \label{fig:exp3_acc}
\end{figure}
\begin{figure}[!t]
  \begin{minipage}[c]{0.5\textwidth}
  \vspace{-30pt}
    \includegraphics[width=\textwidth]{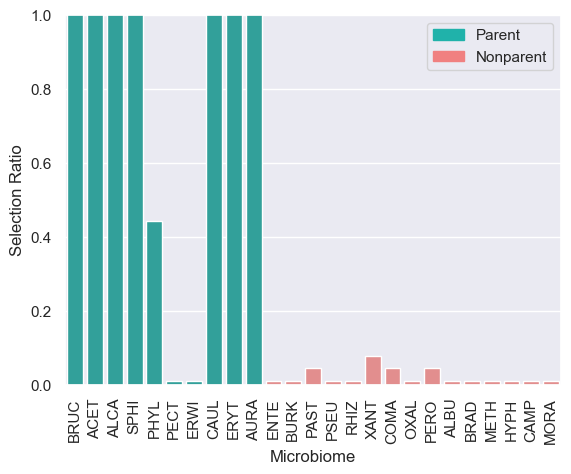}
  \end{minipage}\hfill
  \begin{minipage}[b]{0.45\textwidth}
    \caption{
     The ratio of simulations that each variable is selected by DRCFS to the total number of simulations (30) for linear target $f=f_1$. The ground truth about parental status is given in the legend. DRCDS captures most of the direct causes only with identity feature map $\phi(\boldsymbol{X}) = \boldsymbol{X}$.
    } \label{fig:mic_main_0}
  \end{minipage}
\end{figure}
\textbf{Results.}
Evaluation of the performance of DRCFS and the baselines for linear models ($f = f_1$) are illustrated w,r,t number of nodes in~\Cref{fig:exp1_acc}. DRCFS significantly outperforms the baselines in high dimensions. The same plot for fully nonlinear models (Log-sum-exp $f=f_6$ and geometric mean $f=f_5$) are shown in~\Cref{fig:exp4_csi,fig:exp3_acc}. Note that geometric mean has a highly nonlinear structure that most of the existing models cannot support theoretically. Similar plots for different settings, e.g., hide probability $p_h$, connectivity level $p_c$, linear and nonlinear transformation functions $f$, random variable $\epsilon$ are represented in~\Cref{fig:exp1_f1,fig:exp1_csi,fig:exp2_acc,fig:exp2_csi,fig:exp2_f1,fig:exp4_f1,fig:exp3_csi,fig:exp3_f1} (See~\Cref{sec:add_plots}). DRCFS establishes reasonably good performance w.r.t these settings, even though identity feature map $\phi(\boldsymbol{X}) = \boldsymbol{X}$ is used.

\subsection{Causal Feature Selection for Semi-synthetic Data on Microbiome Abundance}
In this part, to assess the performance of our algorithm with a taste of real-world application, we conduct a semi-synthetic experiment based on microbiome abundance data in plant leaves from~\citet{regalado2020combining}. The dataset contains shotgun sequencing of 275 wild \textit{Arabidopsis Thaliana} leaf microbiomes. This shotgun data provides the ratio of microbial load to plant DNA on the host planet. Microbiome abundance datasets are known to have highly complicated underlying structures with hidden confounders such as the environment, host genetics, and other biological interactions~\citep{lim2021gut, regalado2020combining, yang2022comprehensive}. 

\textbf{Dataset.}
Following the preprocessing, the dataset comprises 625 observations for 25 microorganisms (to serve as confounders $\boldsymbol{X}$) that exhibit the highest level of variations within the dataset. These microorganisms are given in~\Cref{tab:microbiomes}.
The natural target variables of interest for study often include the abundance of bacterial, eukaryotic, and total pathogenic organisms. Guided by this intuition, a subset of 10 confounders is randomly selected as direct causes from the entire set of confounders. Subsequently, target variables are constructed in adherence to Axioms \ref{cond:1}-\ref{cond:3} with different choices of function $f$. 
\textbf{Results.}
The selection ratio by DRCFS for parent and non-parent nodes is illustrated in~\Cref{fig:mic_main_0,fig:mic_main_1} for both linear and nonlinear target functions. Despite a limited number of observations, DRCFS is able to accurately identify most of the direct causes in the linear case and many in the nonlinear case, while maintaining a low false positive rate. Additional plots for other target functions $f$ can be found in~\Cref{fig:mic_appendix} (See~\Cref{sec:add_plots}).
\begin{figure}[!h]
  \begin{minipage}[c]{0.5\textwidth}
  \vspace{-85pt}
    \includegraphics[width=\textwidth]{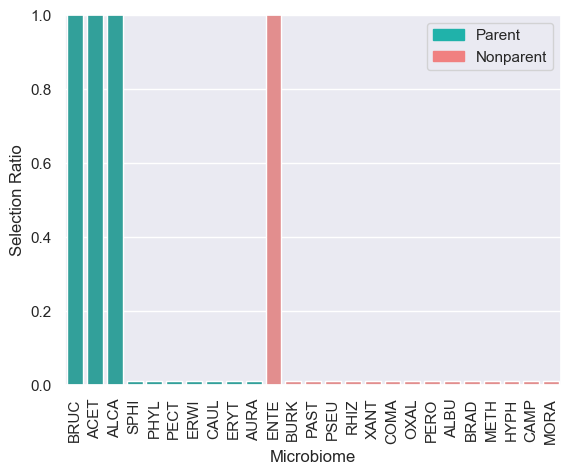}
  \end{minipage}\hfill
  \begin{minipage}[b]{0.45\textwidth}
    \caption{
     The ratio of simulations that each variable is selected by DRCFS to the total number of simulations (30) for Log-sum-exp target $f=f_6$. The ground truth about parental status is given in the legend. Despite the highly complicated structure and limited number of observations, DRCDS captures some of the direct causes only with identity feature map $\phi(\boldsymbol{X}) = \boldsymbol{X}$.
    } \label{fig:mic_main_1}
  \end{minipage}
\end{figure}
\begin{table*}[t]
\centering
\caption{Report on Performance of DRCFS on bnlearn benchmarks. The characteristics of the included networks are provided. To see the performance of DRCFS thoroughly on the underlying causal structure, please refer to \Cref{sec:add_plots}, \Cref{fig:andes,fig:mehra,fig:alarm,fig:arth150}.}
{\scriptsize
\begin{tabular}{l c c c c}
\toprule
\textbf{Category\textbackslash Dataset}  & \textbf{ANDES}  & \textbf{MEHRA}    & \textbf{ALARM}  & \textbf{ARTH150} \\
\midrule
Bayesian Netowkr Type & Discrete & Conditional Linear Gaussian & Discrete & Gaussian \\
Number of nodes & 223 & 24 & 37 & 107 \\
Number of arcs & 338 & 71 & 46 & 150 \\
Number of parameters & 1157 & 324423 & 509 & 364 \\
Average Markov blanket size & 5.61 & 13.75 & 3.51 & 3.35 \\
Average degree & 3.03 & 5.92 & 2.49 & 2.8 \\
Maximum in-degree & 6 & 13 & 4 & 6 \\
\bottomrule
Chosen target name & SNode\_65 & pm2.5 & PRSS & 786\\
\bottomrule
\bottomrule
\textbf{Accuracy} & \textbf{0.973} & \textbf{0.74} & \textbf{0.833} & \textbf{0.915}
\end{tabular}
}
\label{table:bnlearn}
\end{table*}

\subsection{Causal Feature Selection for Bnlearn Benchmarks}

To further demonstrate the performance, scalability, and adaptability of our method across various domains, we consider four additional bnlearn~\citep{bnlearn}\footnote{ \url{https://cran.r-project.org/web/packages/bnlearn/index.html}} benchmarks with different properties. The results and details on these networks are reported in~\Cref{table:bnlearn}. The semi-synthetic datasets that we include are ANDES~\citep{conati1997line} (a very large size discrete Bayesian network), MEHRA~\citep{vitolo2018modeling} (a medium size conditional linear Gaussian Bayesian Network), ALARM~\citep{beinlich1989alarm} (a medium size discrete Bayesian network) and ARTH150~\citep{opgen2007correlation} (a very large size Gaussian Bayesian network). The diagrams of these networks with the inferred direct causes by DRCFS are depicted in \Cref{sec:add_plots}, \Cref{fig:andes,fig:mehra,fig:alarm,fig:arth150}. Notably, our method shows good performance in all instances, with very few false positives and no false negatives. In essence, it effectively captures the entirety of direct causes while maintaining a substantially low incidence of wrongly selected variables.

\section{Discussion}
\textbf{Limitations.} In this work we have focused on the case of learning causal features from observational data rather than full causal discovery. Moreover, while we have shown $\sqrt{n}$-consistency, the expectations in Algorithm \ref{alg} might in practice only be approximated well if enough samples are available to perform the required sample splitting efficiently and without sacrificing approximation quality. Finally, since our method is inspired by double machine learning-based approaches \cite{10.1093/biomet/asaa054}, the focus is on bias rather than variance reduction of estimates. 

\textbf{Conclusion and future work.} We presented DRCFS, a doubly robust feature selection method for identifying causal features even in high-dimensional and nonlinear settings. In extensive experiments including state-of-the-art baselines, we demonstrated that our approach is significantly more accurate across various metrics and across a wide range of settings in synthetic and semi-synthetic datasets. Future work could see the extension of our work, especially to biomedical settings where the evaluation can not be conducted with respect to the ground truth or to representation learning approaches where features first have to be learned before they can be selected. Another potential avenue for future research is the extension to time series data. In the long-term, we hope that these future algorithms as well as the provided feature selection method  will overall enable users to better understand, interpret and explain the results of machine learning-based decision-making. 
\bibliographystyle{abbrvnat}
\bibliography{bibliography.bib}
\newpage
\onecolumn
\renewcommand{\thesection}{\Alph{section}}
\setcounter{section}{0}
\newpage
\onecolumn
\renewcommand{\thesection}{\Alph{section}}
\setcounter{section}{0}
\noindent {\LARGE\textbf{Appendix}}
\section{Structural Causal Models}
\label{appendix:scm}
\begin{definition}[Structural Causal Model (SCM), Definition 2.1 by \citet{10.1214/21-AOS2064}]
\label{def:SCM}
A structural causal model (SCM) is a tuple $\langle \mathtt{I},\mathtt{J}, \boldsymbol{V}, \boldsymbol{U}, \boldsymbol{f}, \mathbb{P}_{\boldsymbol{U}} \rangle $ where
(i) $\mathtt{I}$ is a finite index set of endogenous variables; (ii) $\mathtt{J}$ is a disjoint finite index set of exogenous variables; (iii) $\boldsymbol{V} = \prod_{j \in \mathtt{I}} V_j$ is the product of the domains of the endogenous variables, where each $V_j$ is a standard measurable space; (iv) $\boldsymbol{U} = \prod_{j \in \mathtt{J}} U_j$ is the product of the domains of the exogenous variables, where each $U_j$ is a standard measurable space; (v) $\boldsymbol{f} \colon \boldsymbol{V} \times \boldsymbol{U} \to \boldsymbol{V}$ is a measurable function that specifies the causal mechanism; (vi) $\mathbb{P}_{\boldsymbol{U}} = \prod_{j \in \mathtt{J}} \mathbb{P}_{\boldsymbol{U}_j}$ is a product measure, where $\mathbb{P}_{\boldsymbol{U}_j}$ is a probability measure on $\boldsymbol{U}_j$ for each $j \in \mathtt{J}$.
\end{definition}
In SCMs, the functional relationships between variables are expressed in terms of deterministic equations. This feature allows us to model the cause-effect relationships of the data-generating process (DGP) using \emph{structural equations}. For a given SCM $\langle \mathtt{I},\mathtt{J}, \boldsymbol{V}, \boldsymbol{U}, \boldsymbol{f}, \mathbb{P}_{\boldsymbol{U}} \rangle $ a structural equation specifies an endogenous random variable $V_l$ via a measurable function of the form $V_l = f_{V_l}(\boldsymbol{V}, \boldsymbol{U})$ for all $l \in \mathtt{I}$. A \emph{parent} $i \in \mathtt{I} \cup \mathtt{J}$ of $l$ is any index for which there is no measurable function $g \colon \prod_{j \in \mathtt{I}\setminus \{i\}} V_j \times \boldsymbol{U} \to V_l$ with $f_{V_l} = g$ almost surely. Intuitively, each endogenous variables $V_j$ is specified by its parents together with the exogenous variables, via the structural equations. A structural equations model as in Definition \ref{def:SCM} can be conveniently described with the \emph{causal graph}, a directed graph of the form $\mathcal{G} = (\mathtt{I}\cup \mathtt{J}, \mathcal{E})$. The nodes of the causal graph consist of the entire set of indices for the variables, and the edges are specified by the structural equations, i.e., $\{ j \to l \} \in \mathcal{E}$ iff $j$ is a parent of $l$. Note that the variables in the set $\pa{V_l}$ are indexed by the parent nodes of $l$ in the corresponding graph $\mathcal{G}$.%

\subsection{Interventions}
We define the causal semantics of SCMs, by considering perfect interventions \citep{pearlj}. For a given a SCM as in Definition \ref{def:SCM}, consider a variable $ \boldsymbol{W} \coloneqq \prod_{j \in \mathtt{I}'}V_j$ for a set $\mathtt{I}' \subseteq \mathtt{I}$, and let $\boldsymbol{w}\coloneqq \prod_{j \in \mathtt{I}'}v_j$ be a point of its domain. The perfect intervention $\boldsymbol{W} \gets \boldsymbol{w}$ amounts to replacing the structural equations $V_{j} = f_{V_{j}}(\boldsymbol{V}, \boldsymbol{U})$ with the constant functions $V_{j} \equiv v_{j}$ for all $j \in \mathtt{I}'$. We denote with $V_l \mid do(\boldsymbol{w})$ the variable $V_l$ after performing the intervention. This procedure define a new probability distribution $\pr{}{v_l \mid do(\boldsymbol{w})}$, which we refer to as interventional distribution. This distribution entails the following information: “Given that we have observed
$\boldsymbol{W} = \boldsymbol{w}$, what would $V_l$ have been had we set $do(\boldsymbol{w})$, instead of the value $\boldsymbol{W}$ had actually taken?”.
\section{Unique Solvability} 
\label{appendix:uniquesolv}
We introduce the notions of solvability and unique solvability as defined by \cite{}. These notions describe the existence and uniqueness of measurable solution functions for the structural
equations of a given subset of the endogenous variables. Solvability of an SCM is a sufficient and necessary condition for the existence of a solution of an SCM, and unique solvability implies the uniqueness of the induced observational distribution.

The notion of solvability is defined as follows.
\begin{definition}[Solvability, following Definition 3.1 by \cite{10.1214/21-AOS2064}]
Consider an SCM $\langle \mathtt{I},\mathtt{J}, \boldsymbol{V}, \boldsymbol{U}, \boldsymbol{f}, \mathbb{P}_{\boldsymbol{U}} \rangle $. We say that the SCM is solvable if there exists a measurable mapping $g\colon \mathbf{V}\rightarrow \mathbf{U}$ such that $\boldsymbol{v} = g(\boldsymbol{u}) \Rightarrow \boldsymbol{v} = \boldsymbol f (\boldsymbol{v},\boldsymbol{u})$ almost surely. 
\end{definition}
The unique solvability of an SCM is a stronger notion than mere solvability, and is defined as follows.
\begin{definition}[Unique Solvability, following Definition 3.3 by \cite{10.1214/21-AOS2064}]
Consider an SCM $\langle \mathtt{I},\mathtt{J}, \boldsymbol{V}, \boldsymbol{U}, \boldsymbol{f}, \mathbb{P}_{\boldsymbol{U}} \rangle $. We say that the SCM is uniquely solvable if there exists a measurable mapping $g\colon \mathbf{V}\rightarrow \mathbf{U}$ such that
\begin{equation*}
\boldsymbol{v} = g(\boldsymbol{u}) \Leftrightarrow \boldsymbol{v} = \boldsymbol f (\boldsymbol{v},\boldsymbol{u})
\end{equation*}
almost surely. 
\end{definition}
The unique solvability condition essentially ensures that there exists a measurable solution for the structural equations, and that any possible solution induces the same observational distribution. \citet{10.1214/21-AOS2064} provide the following necessary and sufficient conditions for the unique solvability of an SCM.
\begin{theorem}[Following Theorem 3.6 by \citet{10.1214/21-AOS2064}]
Consider a structural causal model $\langle \mathtt{I},\mathtt{J}, \boldsymbol{V}, \boldsymbol{U}, \boldsymbol{f}, \mathbb{P}_{\boldsymbol{U}} \rangle $. Then, the system of structural equations $\boldsymbol{V} = \boldsymbol f (\boldsymbol{V},\boldsymbol{U})$ has a unique solution almost surely, if and only if the SCM is uniquely solvable. Furthermore, if the SCM is uniquely solvable, then there exists a solution, and all solutions have the same observational distribution.
\end{theorem}
\section{Proof of Lemma \ref{lemma:parents}}
\label{lemmaparents}
\lemmaparents*
\begin{proof}
Note that $X_j \not\in \mathsf{Pa}(Y)$ if and only if $Y\mid do(x_j,\boldsymbol{x}_j^c) = Y\mid do(x'_j,\boldsymbol{x}_j^c)$ for all possible interventional values $x_j, \boldsymbol{x}_j^c$.
Hence, the first part of the claim follows by showing that $Y\mid do(x_j,\boldsymbol{x}_j^c) = Y\mid do(x'_j,\boldsymbol{x}_j^c)$ if and only if $\acde(x_j, x'_j| \boldsymbol{x}_j^c) = 0$ almost surely. If $Y\mid do(x_j,\boldsymbol{x}_j^c) = Y\mid do(x'_j,\boldsymbol{x}_j^c)$, it directly follows that $\acde(x_j, x'_j| \boldsymbol{x}_j^c) = 0$ almost surely, so it only remains to establish the converse.
To this end, suppose that $\acde(x_j, x'_j| \boldsymbol{x}_j^c) = 0$, and define the group $\boldsymbol{Z} = \mathsf{Pa}(Y)$ consisting of all the parents of the outcome.
Note that $\boldsymbol{Z} \subseteq \{X_j\} \cup \boldsymbol{X}_j^c$, since $\{X_j\} \cup \boldsymbol{X}_j^c$ consists of all observed variables of the model.
Hence, the intervention $\{X_j, \boldsymbol{X}_j^c\} \gets \{x_j, \boldsymbol{x}_j^c\}$ defines an intervention on the parents $\boldsymbol{Z} \gets \boldsymbol{z}$, with $\boldsymbol{z}$ a sub-vector of $\{x_j, \boldsymbol{x}_j^c\}$.
Further, we can write the potential outcome as $Y\mid do(x_j, \boldsymbol{x}_j^c) = f(\boldsymbol{z}) + \varepsilon $.
Similarly, the intervention $\{X_j, \boldsymbol{X}_j^c\} \gets \{x_j', \boldsymbol{x}_j^c\}$ defines an intervention of the form $\boldsymbol{Z} \gets \boldsymbol{z}'$, and it follows that $Y\mid do(x_j', \boldsymbol{x}_j^c) = f(\boldsymbol{z}') + \varepsilon$.
Therefore,
\begin{equation}
\label{eq:3}
    f(\boldsymbol{z}) + \expect{\varepsilon} = \expect{Y\mid do(x_j,\boldsymbol{x}_j^c)} = \expect{Y\mid do(x_j',\boldsymbol{x}_j^c)} = f(\boldsymbol{z}')+ \expect{\varepsilon}.
\end{equation}
where the first and the third equalities follow since $\varepsilon$ is exogenous independent noise, and the second equality follows since $\acde(x_j, x'_j| \boldsymbol{x}_j^c) = 0$. From Eq. \eqref{eq:3} we conclude that $Y\mid do(x_j,\boldsymbol{x}_j^c) = Y\mid do(x'_j,\boldsymbol{x}_j^c)$, as claimed.
\end{proof}
\section{Proof of Lemma \ref{lemma:chi}}
\label{lemmachi}
\lemmachi*
\begin{proof}
We first show that the claim follows if
\begin{equation}
\label{eq:2_new}
    \expect{Y \mid do(x_j, \boldsymbol{x}_j^c)} = \expect{Y \mid x_j, \boldsymbol{x}_j^c} \quad \text{a.s.},
\end{equation}
and then we will prove Eq. \eqref{eq:2_new}. To this end, assume that Eq. \eqref{eq:2_new} holds and suppose that $\chi_j = 0$, i.e., $\expect{Y | x_j, \boldsymbol{x}_j^c} = \expect{Y | \boldsymbol{x}_j^c} = \expect{Y | x_j' , \boldsymbol{x}_j^c}$ a.s. Then, 
\begin{equation*}
    \expect{Y\mid do(x_j, \boldsymbol{x}_j^c)} = \expect{Y \mid x_j, \boldsymbol{x}_j^c} = \expect{Y \mid x_j', \boldsymbol{x}_j^c} = \expect{Y\mid do(x_j', \boldsymbol{x}_j^c)} \quad \text{a.s.}
\end{equation*}
Here, the first and third equalities follow from Eq. \eqref{eq:2_new}. It follows that $\acde(x_j, x'_j| \boldsymbol{z}) = 0$ a.s. Conversely, suppose that Eq. \eqref{eq:2_new} holds, and that $\acde(x_j, x'_j| \boldsymbol{z}) = 0$. Then, it holds that $\expect{Y\mid do(x_j, \boldsymbol{x}_j^c)} = \expect{Y\mid do(x_j', \boldsymbol{x}_j^c)}$ a.s., that is,
\begin{equation}
\label{eq:3_new}
\expect{Y\mid do(x_j, \boldsymbol{x}_j^c)} = \mathbb{E}\left [\expect{Y\mid do(x_j, \boldsymbol{x}_j^c)} \mid \boldsymbol{x}_j^c  \right ].
\end{equation}
Hence,
\begin{align*}
    \expect{Y \mid x_j, \boldsymbol{x}_j^c } & = \expect{Y\mid do(x_j, \boldsymbol{x}_j^c)} & [\text{by Eq. \eqref{eq:2_new}}]\\
    & = \mathbb{E} [\expect{Y\mid do(x_j, \boldsymbol{x}_j^c)} \mid \boldsymbol{x}_j^c  ] & [\text{by Eq. \eqref{eq:3_new}}]\\
    & = \mathbb{E}\left [\expect{Y\mid x_j, \boldsymbol{x}_j^c} \mid \boldsymbol{x}_j^c \right ] & [\text{by Eq. \eqref{eq:2_new}}]\\
    & = \expect{Y \mid \boldsymbol{x}_j^c }, & 
\end{align*}
and the claim follows. 

We conclude the proof by showing that Eq. \eqref{eq:2_new} holds. To this end, define the group $\boldsymbol{Z} = \mathsf{Pa}(Y)$ consisting of all the parents of the outcome. Note that $\boldsymbol{Z} \subseteq \{X_j\} \cup \boldsymbol{X}_j^c$, since $\{X_j\} \cup \boldsymbol{X}_j^c$ consists of all observed variables of the model. By Axioms (A)-(C), the outcome can be described as $Y = f(\boldsymbol{Z}) + \varepsilon$, where $\varepsilon$ is independent of $\{X_j\} \cup \boldsymbol{X}_j^c$. Hence by Rule 2 of the do-calculus~\citep[page~85]{pearlj}, $Y\mid do(x_j, \boldsymbol{x}_j^c) \sim Y\mid x_j, \boldsymbol{x}_j^c$, because $Y$ becomes independent of $\{X_j\} \cup \boldsymbol{X}_j^c$ once all arrows from $Z$ to $Y$ are removed from the graph of the DGP. Therefore, $\expect{Y\mid do(x_j,\boldsymbol{x}_j^c)} = \expect{Y\mid x_j,\boldsymbol{x}_j^c}$.
%
\end{proof}

\section{Proof of Lemma \ref{lemma:cond_variance}}
\label{appendix:cond_variance}
\condvariance*
\begin{proof}
Since $\boldsymbol{X}_j^c \subset \boldsymbol{X}$, we have by the tower property of the expectation~\cite{Williams-1991} that
\begin{align*}
    \chi_j & =  \expect{\left (\expect{Y\mid \boldsymbol{X}} - \expect{Y\mid \boldsymbol{X}_j^c} \right )^2} \\
    & = \expect{\expect{\left (\expect{Y\mid \boldsymbol{X}} - \expect{Y\mid \boldsymbol{X}_j^c} \right )^2\mid \boldsymbol{X}_j^c}} \\
    & = \expect{\expect{(\expect{Y\mid \boldsymbol{X}}^2 - \expect{Y\mid \boldsymbol{X}}  \expect{Y\mid \boldsymbol{X}_j^c})\mid \boldsymbol{X}_j^c}} \\
    & = \expect{\expect{Y\mid \boldsymbol{X}}^2} - \expect{\expect{(\expect{Y\mid \boldsymbol{X}}  \expect{Y\mid \boldsymbol{X}_j^c})\mid \boldsymbol{X}_j^c}}\\
    & = \expect{\expect{Y\mid \boldsymbol{X}}^2} - \expect{\expect{Y\mid \boldsymbol{X}_j^c}^2}\\
    & = \expect{Y\expect{Y\mid \boldsymbol{X}}} - \expect{Y\expect{Y\mid \boldsymbol{X}_j^c}}.
\end{align*}
The claim follows since $m_0(\boldsymbol{V}; g_0) = Y\expect{Y\mid \boldsymbol{X}}$ and $m_j(\boldsymbol{V}; h_0) = Y\expect{Y\mid \boldsymbol{X}_j^c}$.
\end{proof}

\section{Experimental Setup}

\subsection{Evaluation Metrics} 
\label{sec:eval_metrics}
We use Accuracy, F1 Score (harmonic mean of precision and sensitivity) and Critical Success Index (CSI) as metrics to assess the performance of the algorithms. Given the number of true positives TP, true negatives TN, false positives FP, and false negatives FN, these metrics are defined as,
\begin{multicols}{3}
    \begin{enumerate}[label={$\bullet$},itemsep=0pt,topsep=0pt]
        \item $\textrm{ACC} = \ffrac{\textrm{TP} + \textrm{TN}}{\textrm{TP} + \textrm{TN} + \textrm{FP} + \textrm{FN}}$
        \item $\textrm{F1} = \ffrac{2 \textrm{TP}}{2 \textrm{TP} + \textrm{FP} + \textrm{FN}}$
        \item $\textrm{CSI} = \ffrac{\textrm{TP}}{\textrm{TP} + \textrm{FP} + \textrm{FN}}$
    \end{enumerate}
\end{multicols}
The reason that we consider F1 score, and CSI is the emphasis placed on the number of true positives within their calculations.
%

\subsection{Designs for Transformation Function $f$} \label{sec:design_f}
Distributions over the different choices of the transformation function $f$ are used in the experiments to generate the variables based on $X = f(\pa{X}) + \varepsilon\textrm{, for a subset }\pa{X}\subseteq \boldsymbol{X}$:

\begin{multicols}{2}
\begin{enumerate}[label={$\bullet$},itemsep=0pt,topsep=0pt]
   \item \textbf{Linear}:
    \begin{equation*}
        f_1(\pa{X}) = a\sum_{X^\prime \in \pa{X}} X^\prime + b,
    \end{equation*}
    with $a = 0.5$ and $b = 0$, unless stated exactly otherwise.
    \item \textbf{Sum-sqrt}:
    \begin{equation*}
        f_2(\pa{X}) = a\sum_{X^\prime \in \pa{X}} \sqrt{|X^\prime|} + b,
    \end{equation*}
    with $a = 0.5$ and $b = 0$, unless stated exactly otherwise.
   \item \textbf{Sum-sine}:
    \begin{equation*}
        f_3(\pa{X}) = a\sum_{X^\prime \in \pa{X}} \sin({c.X^\prime}) + b,
    \end{equation*}
    with $a = 1, b = 0$ and $c = 0.5$, unless stated exactly otherwise.
    \item \textbf{Sum-tanh}:
    \begin{equation*}
        f_4(\pa{X}) = a\sum_{X^\prime \in \pa{X}} \tanh({c.X^\prime}) + b,
    \end{equation*}
    with $a = 1, b = 0$ and $c = 2$, unless stated exactly otherwise.
    \item \textbf{Geometric mean}:
    \begin{equation*}
        f_5(\pa{X}) = a {\prod_{X^\prime \in \pa{X}} |{X^\prime}|}^{\ffrac{1}{\mathbf{card}(\pa{X})}} + b,
    \end{equation*}
    with $a = 3$ and $b = 0.1$, unless stated exactly otherwise.
    \item \textbf{Log-sum-exp}:
    \begin{equation*}
        f_6(\pa{X}) = a \log{(\sum_{X^\prime \in \pa{X}} e^{X^\prime})} + b,
    \end{equation*}
    with $a = 1$ and $b = \log 2$, unless stated exactly otherwise.
    \item \textbf{Sqrt-sum}:
    \begin{equation*}
        f_7(\pa{X}) = a\sqrt{|\sum_{X^\prime \in \pa{X}} X^\prime|} + b,
    \end{equation*}
    with $a = 1$ and $b = 0$, unless stated exactly otherwise.
\end{enumerate}
\end{multicols}

\subsection{Semi-synthetic Data on
Microbiome Abundance}

The 25 covariates that demonstrate the highest variations in microbiome abundance within the dataset provided by~\citet{regalado2020combining} are listed in~\Cref{tab:microbiomes}. These covariates subsume a diverse set of Bacteria/Eukaryote groups that are common in leaves, soil, and water.
\begin{table}[hb]
\centering
\caption{25 covariates of the microbiome abundance with highest variations within the dataset provided by~\citet{regalado2020combining}.}
\begin{tabular}{l|c}
\multicolumn{1}{c}{\textbf{Microbiome}} & \multicolumn{1}{c}{\textbf{Abbreviation}} \\ \hline
Enterobacteriaceae     & ENTE  \\ \hline
Burkholderiaceae        & BURK     \\ \hline
Pasteurellaceae         & PAST     \\ \hline
Brucellaceae          & BRUC   \\ \hline
Pseudomonadaceae      & PSEU   \\ \hline
Acetobacteraceae      &  ACET   \\ \hline
Alcaligenaceae     &  ALCA   \\ \hline
Rhizobiaceae      &  RHIZ  \\ \hline
Sphingomonadaceae     &  SPHI   \\ \hline
\end{tabular}
\hfill\begin{tabular}{l|c}
\multicolumn{1}{c}{\textbf{Microbiome}} & \multicolumn{1}{c}{\textbf{Abbreviation}} \\ \hline
Xanthomonadaceae     &  XANT   \\ \hline
Comamonadaceae      &  COMA   \\ \hline
Phyllobacteriaceae     &  PHYL   \\ \hline
Oxalobacteraceae      &  OXAL   \\ \hline
Peronosporaceae      &  PERO   \\ \hline
Albuginaceae      &  ALBU   \\ \hline
Pectobacteriaceae      &  PECT   \\ \hline
Bradyrhizobiaceae      &  BRAD   \\ \hline
\end{tabular}
\hfill
\begin{tabular}{l|c}
\multicolumn{2}{c}{}  \\ 
\multicolumn{1}{c}{\textbf{Microbiome}} & \multicolumn{1}{c}{\textbf{Abbreviation}} \\ \hline
Erwiniaceae      &  ERWI   \\ \hline
Caulobacteraceae      &  CAUL   \\ \hline
Methylobacteriaceae      &  METH   \\ \hline
Hyphomicrobiaceae      &  HYPH   \\ \hline
Erythrobacteraceae      &  ERYT   \\ \hline
Campylobacteraceae      &  CAMP  \\ \hline
Aurantimonadaceae      &  AURA \\ \hline
Moraxellaceae      &  MORA   \\ \hline
\end{tabular}
\label{tab:microbiomes}
\end{table}

\newpage
\section{Additional Plots} \label{sec:add_plots}
In this section, additional plots for the experiments are provided.

\begin{figure}[h]
  \centering \includegraphics[width=\columnwidth]{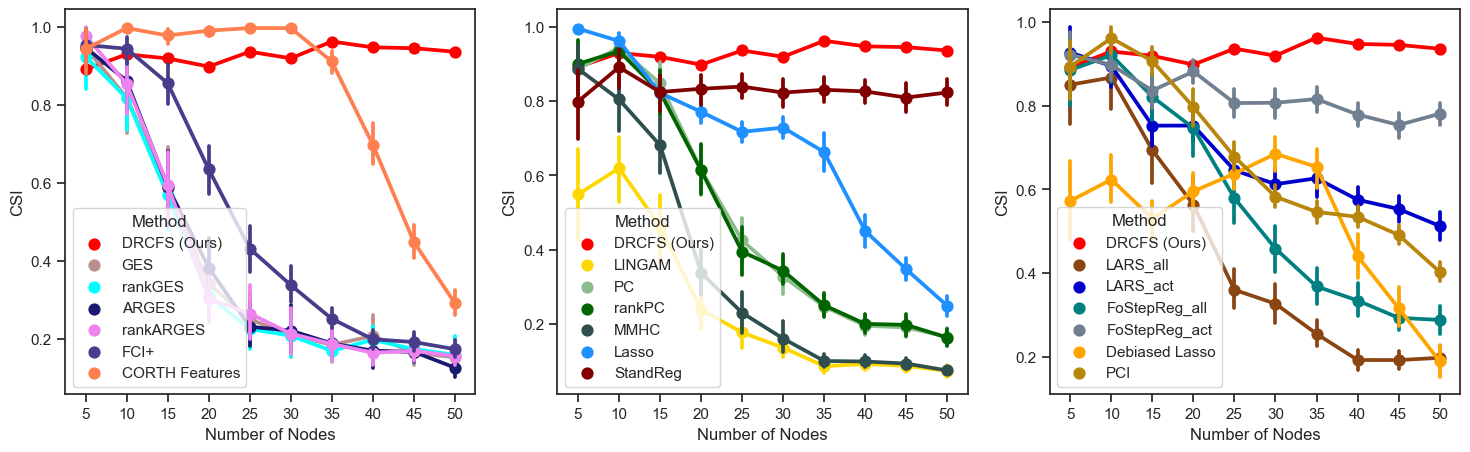}
  \caption{Performance (CSI) of the algorithms w.r.t. number of nodes $m$ for fully linear causal structures ($f = f_1$  with probability 1), where $p_s = 0.3$, $p_h = 0$, and $\epsilon \sim \mathcal{N}(0,1)$. Each case is averaged over 50 simulations. ForestRiesz with identity feature map $\phi(\boldsymbol{X}) = \boldsymbol{X}$ has been used in these experiments. DRCFS's performance shows stability even for large graphs while the baselines suffer in high dimensions.} \label{fig:exp1_csi}
  \vspace{45pt}
\end{figure}

\begin{figure}[h]
  \centering \includegraphics[width=\columnwidth]{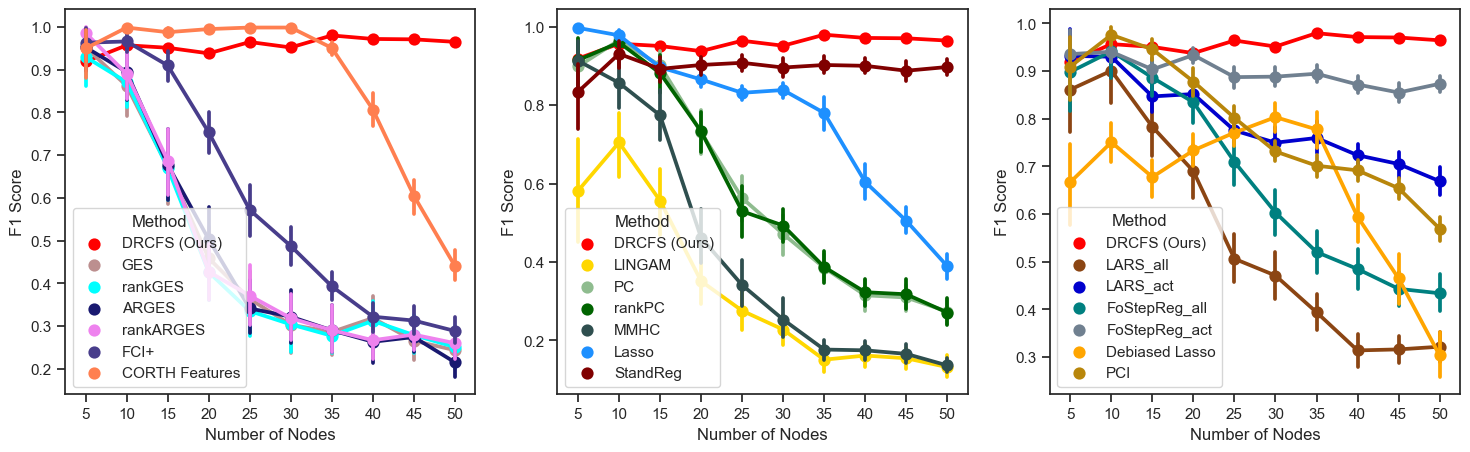}
  \caption{Performance (F1 score) of the algorithms w.r.t. number of nodes $m$ for fully linear causal structures ($f = f_1$  with probability 1), where $p_s = 0.3$, $p_h = 0$, and $\epsilon \sim \mathcal{N}(0,1)$. Each case is averaged over 50 simulations. ForestRiesz with identity feature map $\phi(\boldsymbol{X}) = \boldsymbol{X}$ has been used in these experiments. DRCFS's performance shows stability even for large graphs while the baselines suffer in high dimensions.} \label{fig:exp1_f1}
\end{figure}

\begin{figure}[h]
  \centering \includegraphics[width=\columnwidth]{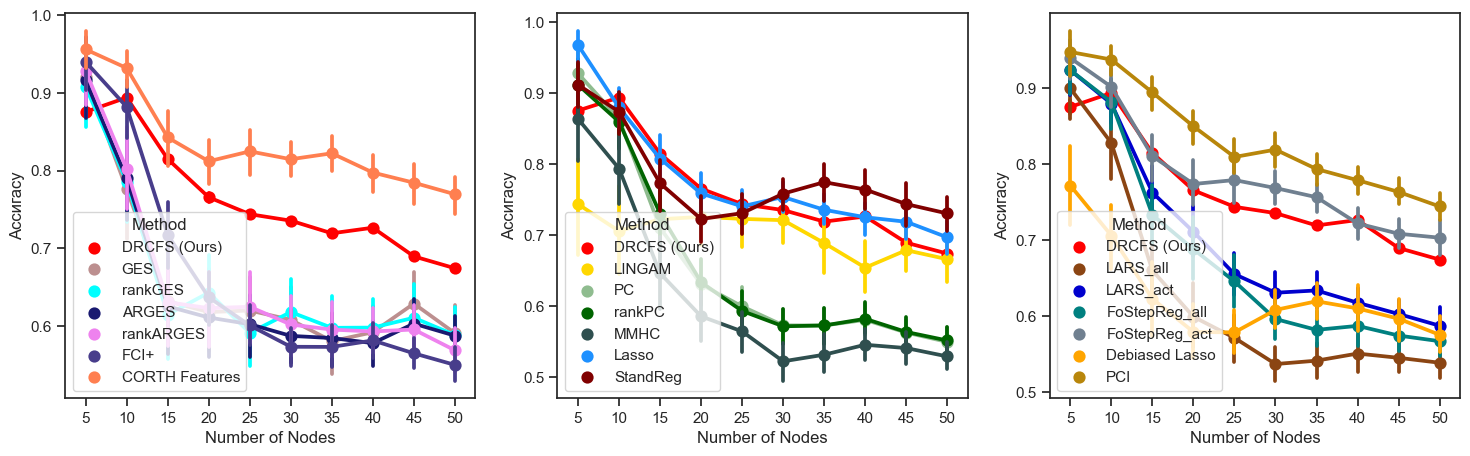}
  \caption{Performance (Accuracy) of the algorithms w.r.t. number of nodes $m$ for causal structures defined as linear combinations of nonlinear functions ($f = f_2$  with probability 0.5, $f = f_3$  with probability 0.25, $f = f_4$  with probability 0.25), where $p_s = 0.5$, $p_h = 0$, and $\epsilon \sim \beta(2,5)$. Each case is averaged over 50 simulations. ForestRiesz with identity feature map $\phi(\boldsymbol{X}) = \boldsymbol{X}$ has been used in these experiments. DRCFS shows reasonably good performance, even though the identity feature map is used. Nevertheless, alternative feature maps taking into account prior domain knowledge of the dataset could be used. CORTH Features dominates others because in the regime that the noise is defined, summations of non-linear terms in DGP fro the target variable act approximately linear, hence, these results are consistent with~\citet{soleymani2022causal}.} \label{fig:exp2_acc}
\end{figure}

\begin{figure}[h]
  \centering \includegraphics[width=\columnwidth]{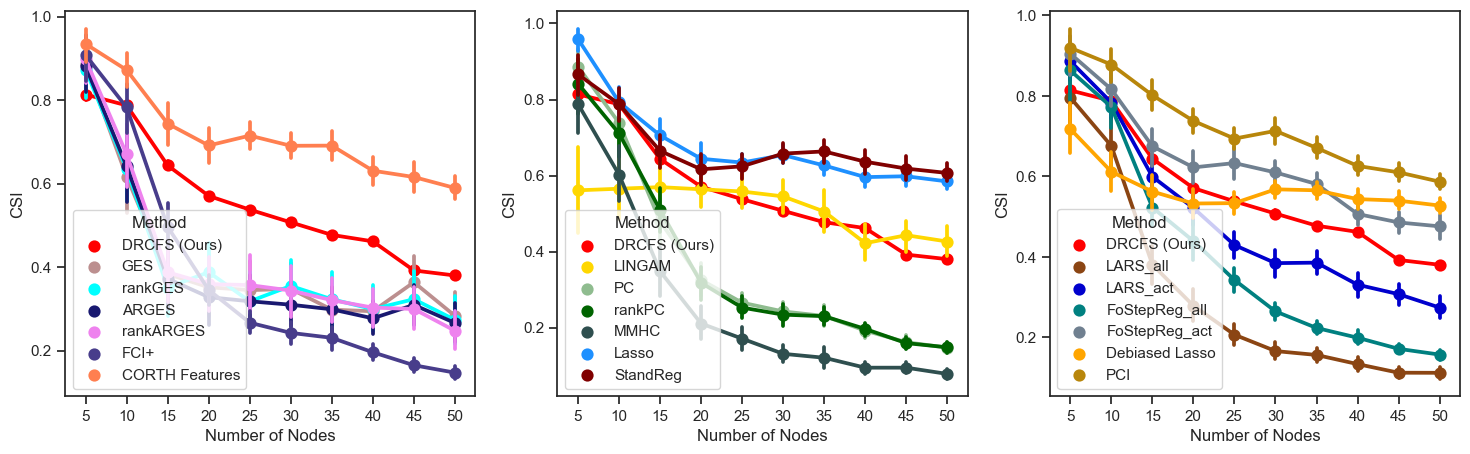}
  \caption{Performance (CSI) of the algorithms w.r.t. number of nodes $m$ for causal structures defined as linear combinations of nonlinear functions ($f = f_2$  with probability 0.5, $f = f_3$  with probability 0.25, $f = f_4$  with probability 0.25), where $p_s = 0.5$, $p_h = 0$, and $\epsilon \sim \beta(2,5)$. Each case is averaged over 50 simulations. ForestRiesz with identity feature map $\phi(\boldsymbol{X}) = \boldsymbol{X}$ has been used in these experiments. DRCFS shows reasonably good performance, even though the identity feature map  is used. Nevertheless, alternative feature maps taking into account prior domain knowledge of the dataset could be used. CORTH Features dominates others because in the regime that the noise is defined, summations of non-linear terms in DGP fro the target variable act approximately linear, hence, these results are consistent with~\citet{soleymani2022causal}.} \label{fig:exp2_csi}
\end{figure}

\begin{figure}
  \centering \includegraphics[width=\columnwidth]{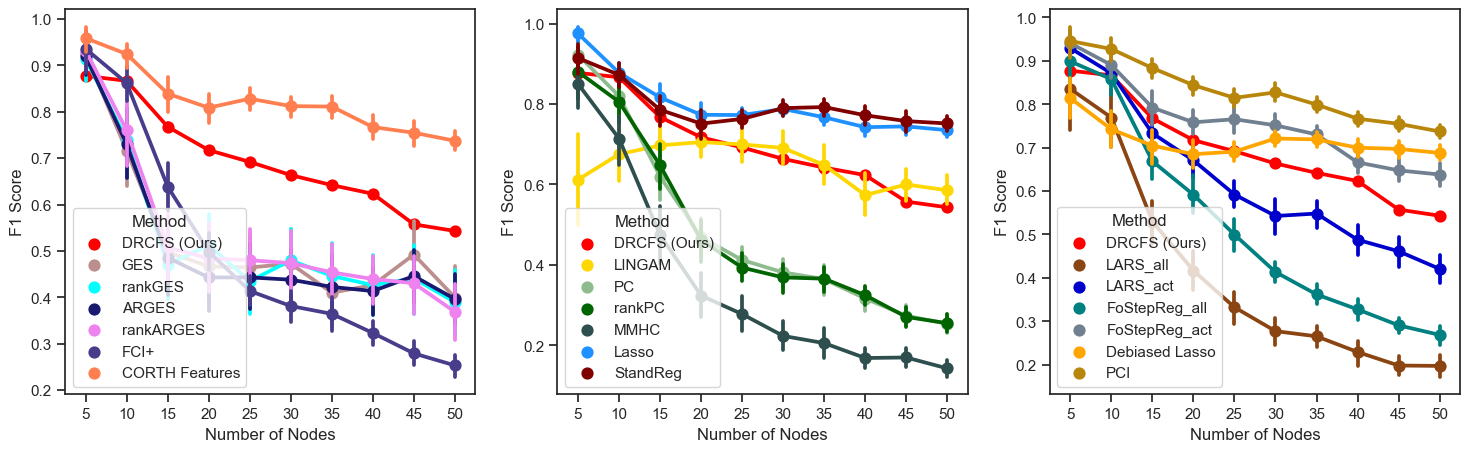}
  \caption{Performance (F1 score) of the algorithms w.r.t. number of nodes $m$ for causal structures defined as linear combinations of nonlinear functions ($f = f_2$  with probability 0.5, $f = f_3$  with probability 0.25, $f = f_4$  with probability 0.25), where $p_s = 0.5$, $p_h = 0$, and $\epsilon \sim \beta(2,5)$. Each case is averaged over 50 simulations. ForestRiesz with identity feature map $\phi(\boldsymbol{X}) = \boldsymbol{X}$ has been used in these experiments. DRCFS shows reasonably good performance, even though the identity feature map is used. Nevertheless, alternative feature maps taking into account prior domain knowledge of the dataset could be used. CORTH Features dominates others because in the regime that the noise is defined, summations of non-linear terms in DGP fro the target variable act approximately linear, hence, these results are consistent with~\citet{soleymani2022causal}.} \label{fig:exp2_f1}
\end{figure}

\begin{figure}[t] 
  \centering \includegraphics[width=\columnwidth]{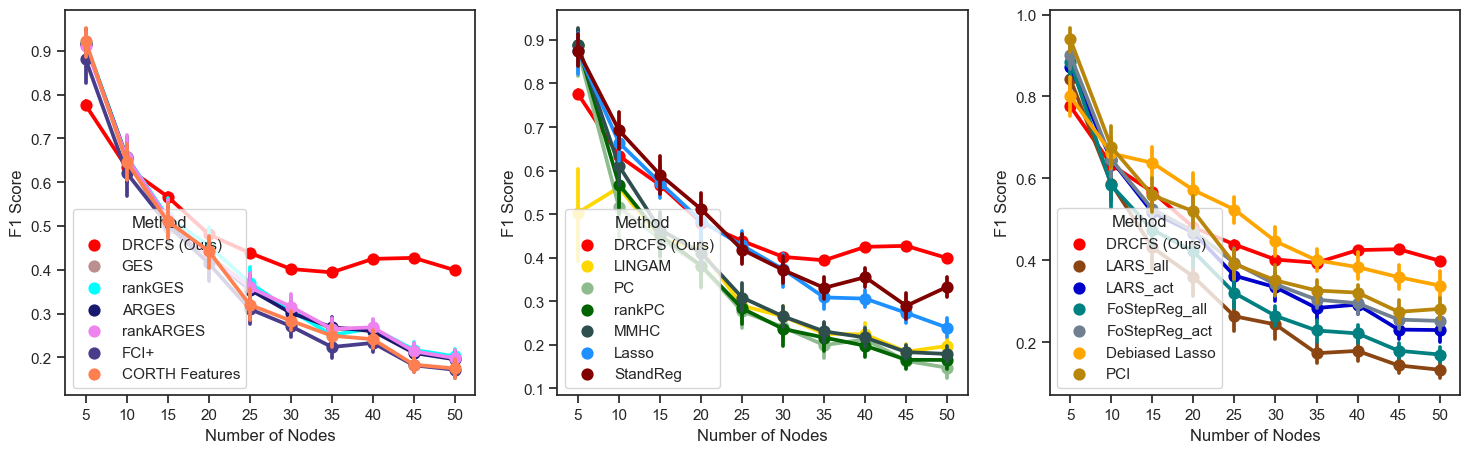}
  \caption{Performance (F1 score) of the algorithms w.r.t. number of nodes $m$ for fully Log-sum-exp causal structures ($f = f_6$ with probability 1), where $p_s = 0.5$, $p_h = 0.1$, and $\epsilon \sim \mathcal{N}(0,1)$. Each case is averaged over 50 simulations. We use ForestRiesz with identity feature map $\phi(\boldsymbol{X}) = \boldsymbol{X}$. DRCFS's performance shows stability even for large graphs while the baselines suffer in high dimensions. }
  \label{fig:exp4_f1}
\end{figure}

\begin{figure}[!h] 
  \centering \includegraphics[width=\columnwidth]{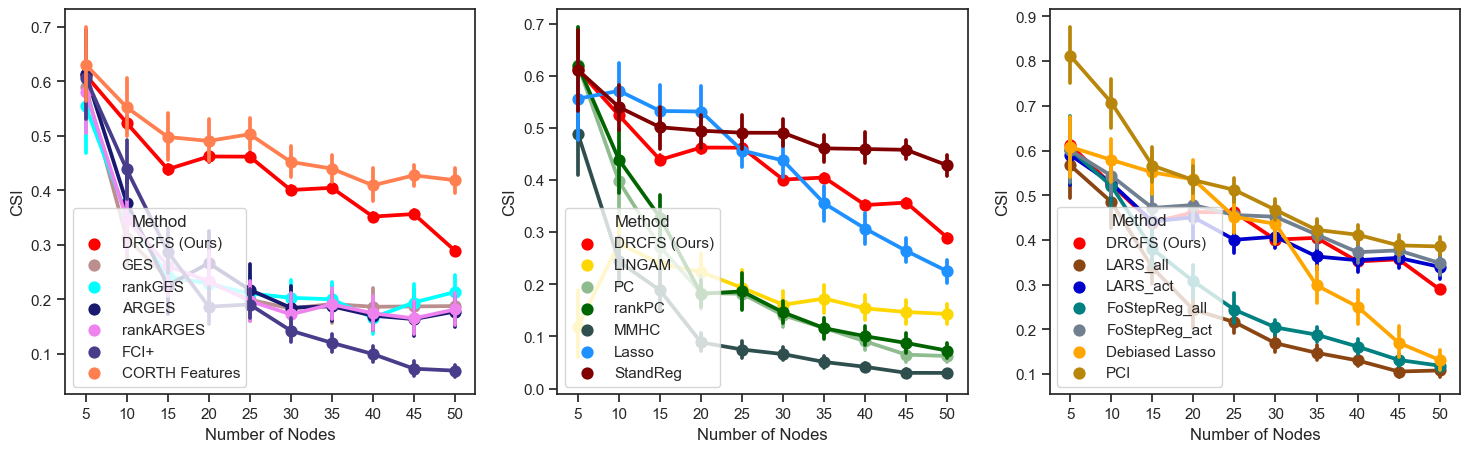}
  \caption{Performance (CSI) of the algorithms w.r.t. number of nodes $m$ for causal structures with geometric mean relationship ($f = f_5$ with probability 0.8, $f = f_1$ with probability 0.2), where $p_s = 0.5$, $p_h = 0$, and $\epsilon \sim \mathcal{N}(0,1)$. Each case is averaged over 50 simulations. We use ForestRiesz with identity feature map $\phi(\boldsymbol{X}) = \boldsymbol{X}$. DRCFS's performance shows stability even for large graphs while the baselines suffer in high dimensions.}
  \label{fig:exp3_csi}
\end{figure}

\begin{figure}[!h] 
  \centering \includegraphics[width=\columnwidth]{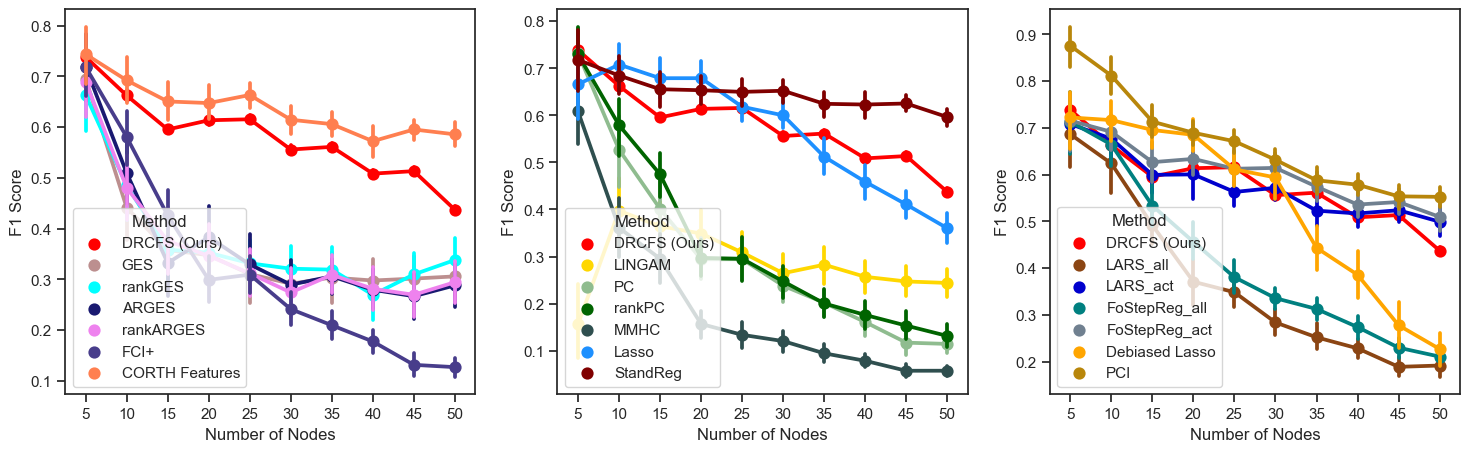}
  \caption{Performance (F1 Score) of the algorithms w.r.t. number of nodes $m$ for causal structures with geometric mean relationship ($f = f_5$ with probability 0.8, $f = f_1$ with probability 0.2), where $p_s = 0.5$, $p_h = 0$, and $\epsilon \sim \mathcal{N}(0,1)$. Each case is averaged over 50 simulations. We use ForestRiesz with identity feature map $\phi(\boldsymbol{X}) = \boldsymbol{X}$. DRCFS's performance shows stability even for large graphs while the baselines suffer in high dimensions.}
  \label{fig:exp3_f1}
\end{figure}

\begin{figure}
\begin{minipage}[c]{0.45\textwidth}
\vspace{-15pt}
      \centering \includegraphics[width=\textwidth]{0.png}
      \caption{Linear target $f=f_1$. The ground truth about parental status is given in the legend. Despite the highly complicated structure, DRCDS captures most of the direct causes only with identity feature map $\phi(\boldsymbol{X}) = \boldsymbol{X}$.}
\end{minipage}
    \hfill
\begin{minipage}[c]{0.45\textwidth}
      \centering \includegraphics[width=\textwidth]{1.png}
      \caption{Log-sum-exp target $f=f_6$. The ground truth about parental status is given in the legend. Despite the highly complicated structure and limited number of observations, DRCDS captures some of the direct causes only with identity feature map $\phi(\boldsymbol{X}) = \boldsymbol{X}$.}
\end{minipage}
    \hfill
\begin{minipage}[c]{0.45\textwidth}
\vspace{15pt}
      \centering \includegraphics[width=\textwidth]{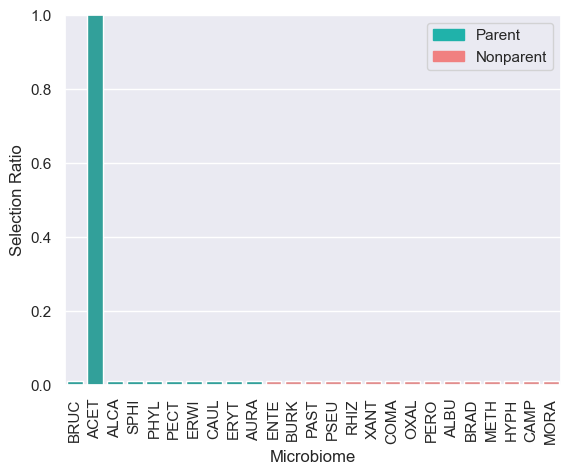}
      \caption{Sqrt-sum target $f=f_7$. The ground truth about parental status is given in the legend. Despite the highly complicated structure and limited number of observations, DRCDS captures some of the direct causes only with identity feature map $\phi(\boldsymbol{X}) = \boldsymbol{X}$.}
\end{minipage}
    \hfill
\begin{minipage}[c]{0.45\textwidth}
\vspace{15pt}
      \centering \includegraphics[width=\textwidth]{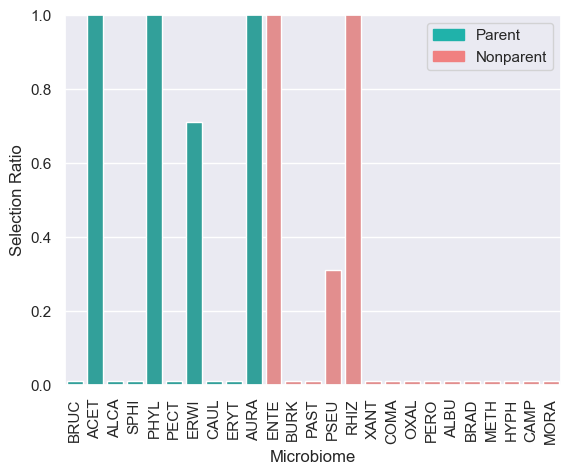}
      \caption{Sum-tanh target $f=f_4$. The ground truth about parental status is given in the legend. Despite the highly complicated structure and limited number of observations, DRCDS captures some of the direct causes only with identity feature map $\phi(\boldsymbol{X}) = \boldsymbol{X}$.}
\end{minipage}
    \caption{The ratio of simulations that each variable is selected by DRCFS to the total number of simulations (30) for different target functions $f$.} \label{fig:mic_appendix}
  \end{figure}

\begin{figure*}[!h] 
  \centering \includegraphics[width=\columnwidth]{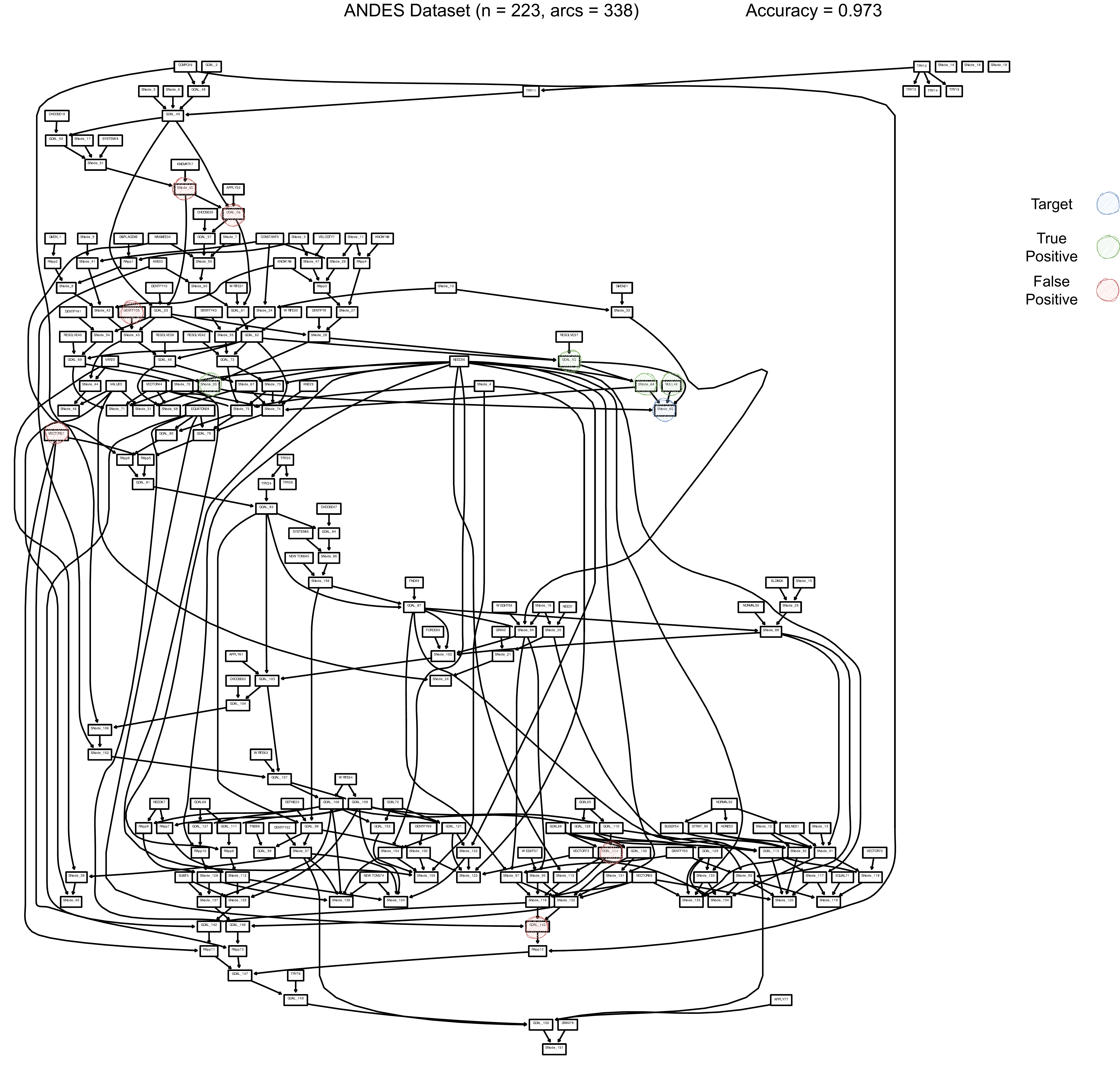}
  \caption{Causal Structure of ANDES benchmark~\citep{bnlearn} and DRCFS's inferred causes. ANDES is a very large size discrete Bayesian network. DRCFS has good performance with very few false positives and no false negatives.}
  \label{fig:andes}
\end{figure*}

\begin{figure*}[!h] 
  \centering \includegraphics[width=\columnwidth]{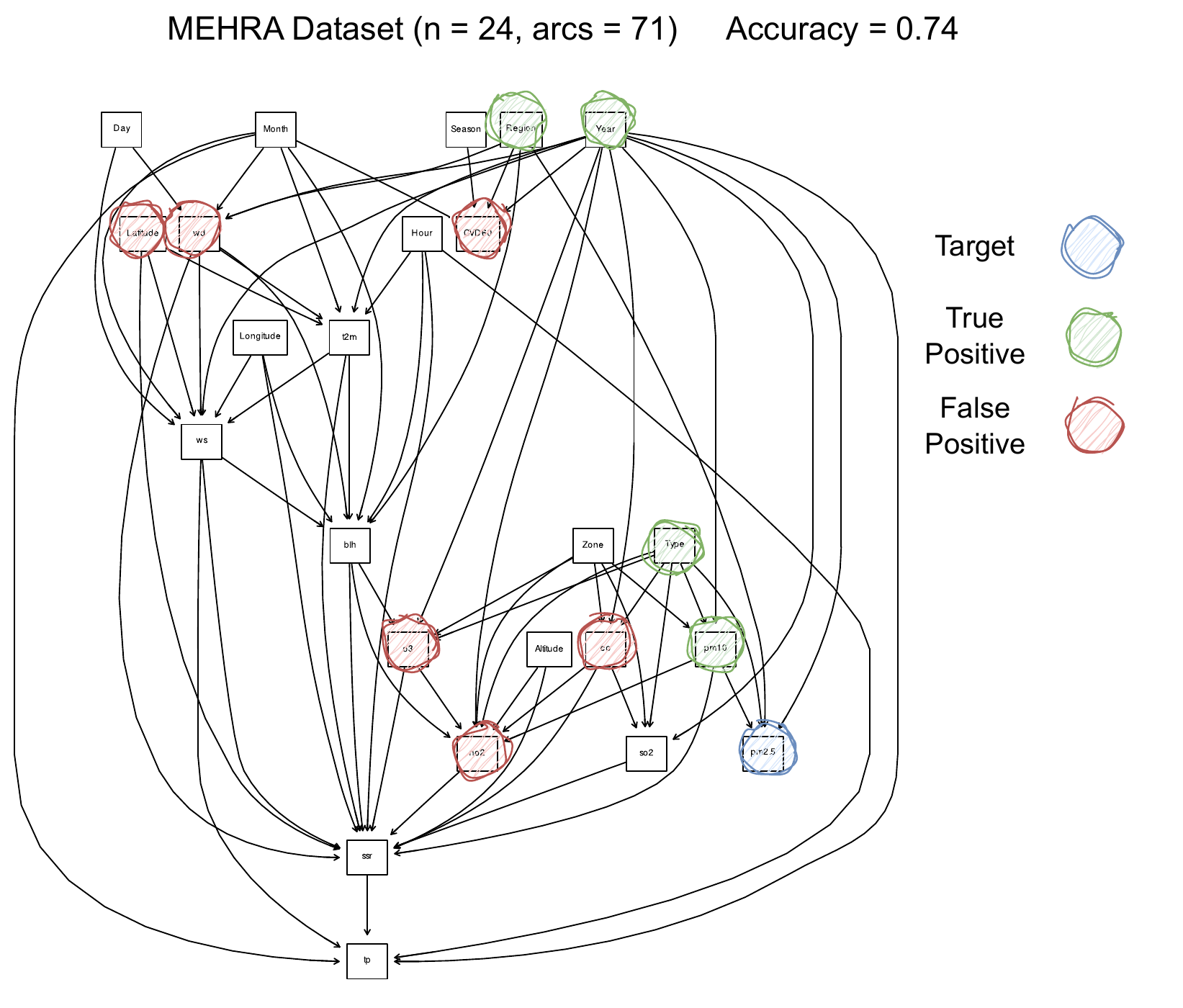}
  \caption{Causal Structure of MEHRA benchmark~\citep{bnlearn} and DRCFS's inferred causes. ANDES is a medium size conditional linear Gaussian Bayesian network. DRCFS has good performance with very few false positives and no false negatives.}
  \label{fig:mehra}
\end{figure*}

\begin{figure*}[!h] 
  \centering \includegraphics[width=\columnwidth]{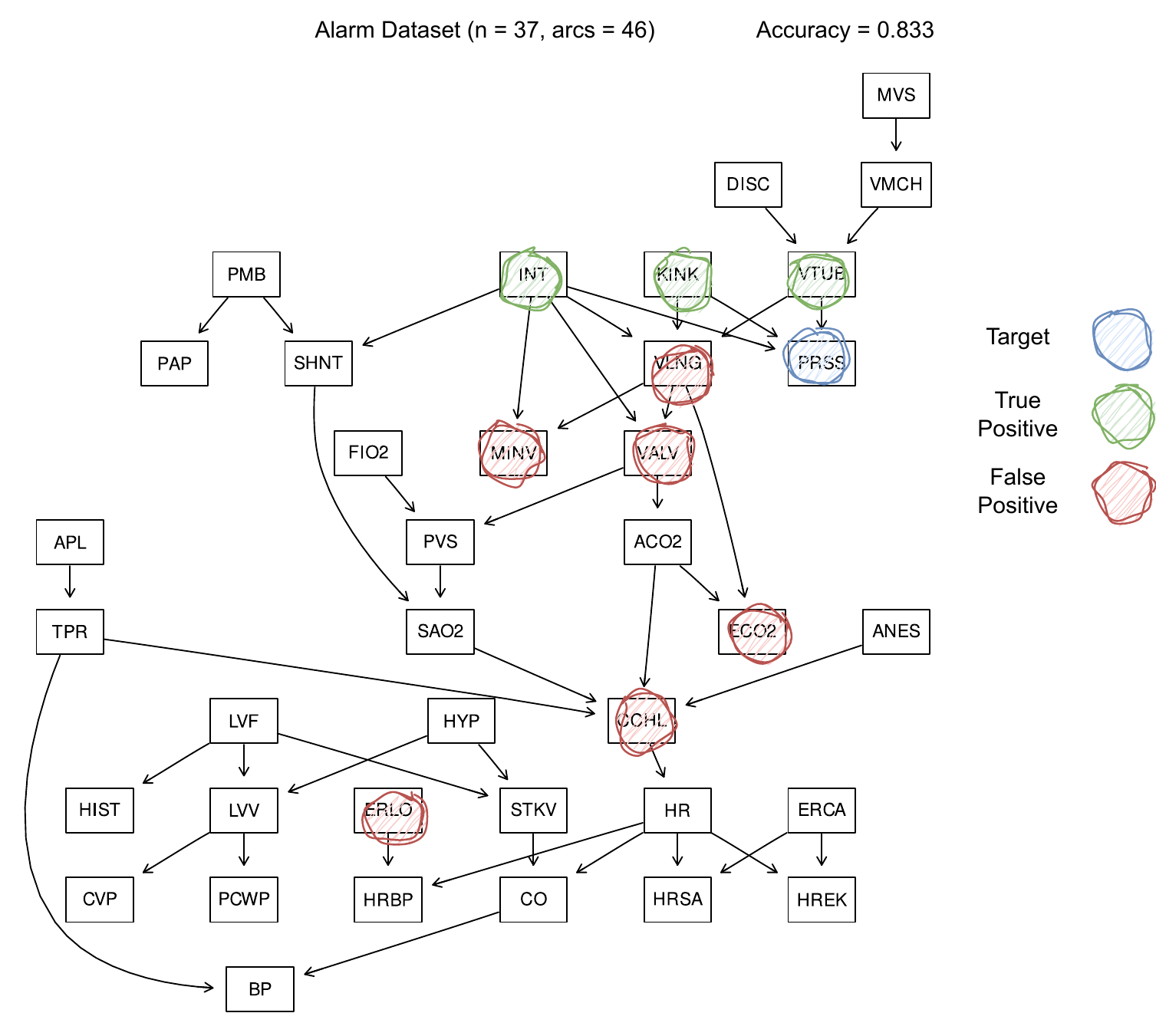}
  \caption{Causal Structure of ALARM benchmark~\citep{bnlearn} and DRCFS's inferred causes. ANDES is a medium size discrete Bayesian Network. DRCFS has good performance with very few false positives and no false negatives.}
  \label{fig:alarm}
\end{figure*}

\begin{figure*}[!h] 
  \centering \includegraphics[width=\columnwidth]{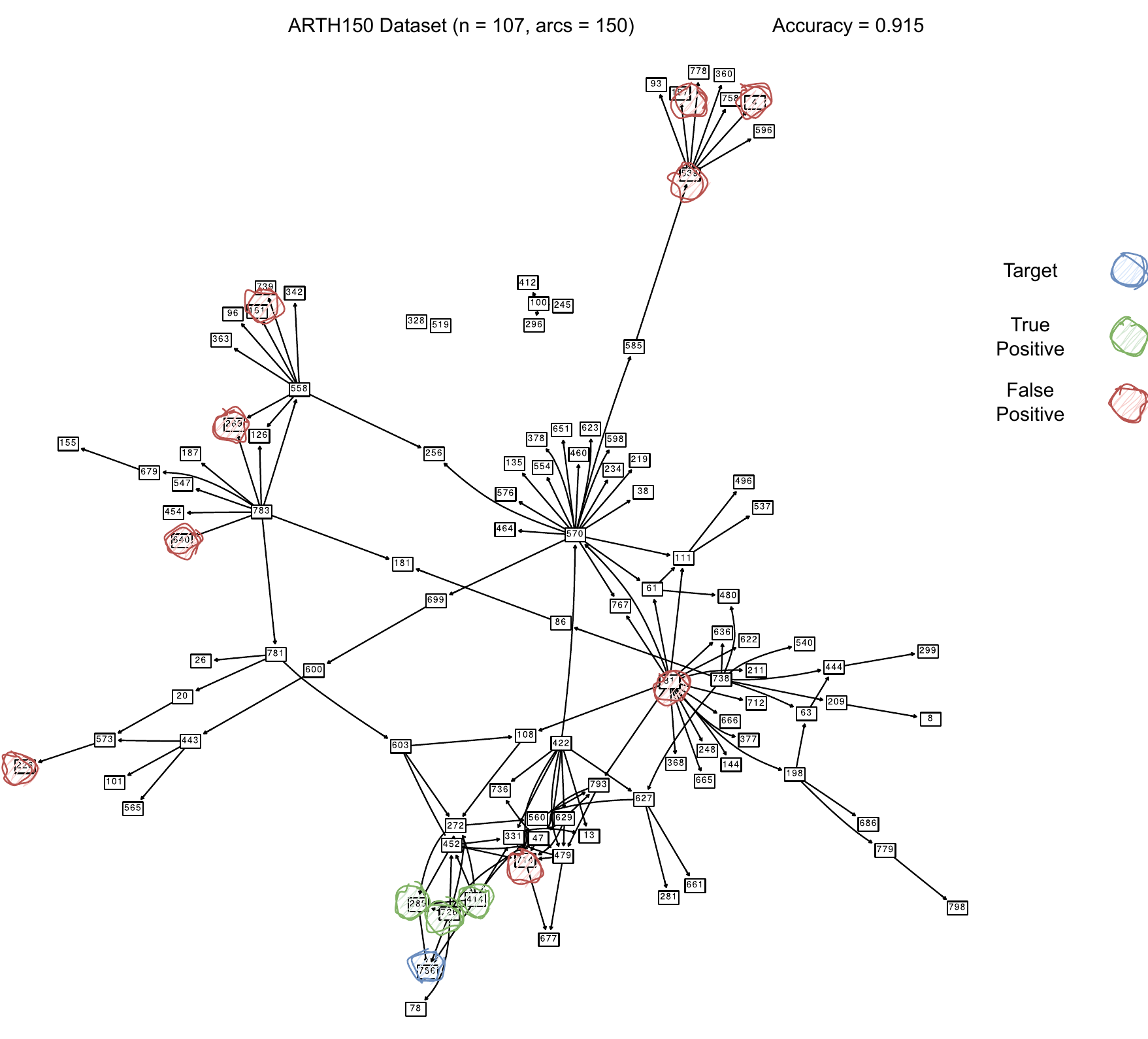}
  \caption{Causal Structure of ARTH150 benchmark~\citep{bnlearn} and DRCFS's inferred causes. ANDES is a very large size Gaussian Bayesian network. DRCFS has good performance with very few false positives and no false negatives.}
  \label{fig:arth150}
\end{figure*}
\end{document}